\documentclass[letterpaper]{article} 
\usepackage{aaai24}  
\usepackage{times}  
\usepackage{helvet}  
\usepackage{courier}  
\usepackage[hyphens]{url}  
\usepackage{graphicx} 
\usepackage{natbib}  
\usepackage{caption} 
\usepackage{algorithm}
\usepackage{algorithmic}
\usepackage{amsfonts}
\usepackage{amsmath}

\usepackage{amsfonts}

\usepackage{amsthm}
\usepackage{booktabs}
\usepackage{multirow}
\usepackage{enumitem}
\usepackage{xcolor}

\newtheorem{definition}{Definition}
\newtheorem{proposition}{Proposition}

\usepackage{graphicx}

\usepackage{newfloat}
\usepackage{listings}
\DeclareCaptionStyle{ruled}{labelfont=normalfont,labelsep=colon,strut=off} 
\floatstyle{ruled}
\newfloat{listing}{tb}{lst}{}
\floatname{listing}{Listing}
\title{CGS-Mask: Making Time Series Predictions Intuitive for All}
\author{
    Feng Lu\equalcontrib\textsuperscript{\rm 1},
    Wei Li\equalcontrib\textsuperscript{\rm 2},
    Yifei Sun\equalcontrib\textsuperscript{\rm 1},
    Cheng Song\textsuperscript{\rm 1},
    Yufei Ren\textsuperscript{\rm 3},
    Albert Y. Zomaya\textsuperscript{\rm 2}
}
\affiliations{
    \textsuperscript{\rm 1}School of Computer Science and Technology, Huazhong University of Science and Technology, China\\
    \textsuperscript{\rm 2}The Australia-China Joint Research Centre for Energy Informatics and Demand Response Technologies, Centre for Distributed and High Performance Computing, School of Computer Science, The University of Sydney, Australia\\
    \textsuperscript{\rm 3}Tongji Hospital, Tongji Medical College, Huazhong University of Science and Technology, China\\
    
    \{lufeng,sunyifei,songch\}@hust.edu.cn, \{weiwilson.li,albert.zomaya\}@sydney.edu.au, yfren@tjh.tjmu.edu.cn

}

\usepackage{bibentry}

\begin{document}

\maketitle

\begin{abstract}
Artificial intelligence (AI) has immense potential in time series prediction, but most explainable tools have limited capabilities in providing a systematic understanding of important features over time. These tools typically rely on evaluating a single time point, overlook the time ordering of inputs, and neglect the time-sensitive nature of time series applications. These factors make it difficult for users, particularly those without domain knowledge, to comprehend AI model decisions and obtain meaningful explanations. We propose CGS-Mask, a post-hoc and model-agnostic cellular genetic strip mask-based saliency approach to address these challenges. CGS-Mask uses consecutive time steps as a cohesive entity to evaluate the impact of features on the final prediction, providing binary and sustained feature importance scores over time. Our algorithm optimizes the mask population iteratively to obtain the optimal mask in a reasonable time. We evaluated CGS-Mask on synthetic and real-world datasets, and it outperformed state-of-the-art methods in elucidating the importance of features over time.  According to our pilot user study via a questionnaire survey, CGS-Mask is the most effective approach in presenting easily understandable time series prediction results, enabling users to comprehend the decision-making process of AI models with ease.
\end{abstract}

\section{Introduction}

Artificial intelligence (AI) is increasingly being used for time series prediction, particularly in fields like healthcare~\cite{health}, physics~\cite{physics}, energy~\cite{energy}, and sensor data~\cite{sensor}. Ensuring that users can quickly comprehend the reasons behind AI decisions is crucial, particularly in time-sensitive applications where AI is employed to aid in decision-making. Explanation methods are designed to investigate the factors behind a decision and expose any biases or unintended effects of AI models. The transparency and explainability provided by these methods increase human trust in AI models, allowing for more extensive and profound use of AI~\cite{kanamori2020dace, albini2020relation, songcheng}.

Saliency methods aim to explain the importance of input features for predictions, which is crucial for building trust among stakeholders~\cite{adebayo2018sanity,sunyifei}. Perturbation-based methods~\cite{ivanovs2021perturbation} like Feature Occlusion (FO)~\cite{suresh2017clinical} and RISE~\cite{Petsiuk2018rise}, perturb inputs and compare the resulting outcomes to provide explanations. Gradient-based methods like Integrated Gradient~\cite{sundararajan2017axiomatic}, DeepLIFT~\cite{shrikumar2017learning}, and GradSHAP~\cite{lundberg2018explainable} use feature representation gradients for explanation. Attention mechanisms are used with saliency methods~\cite{vaswani2017attention,kwon2018retainvis,gomez2022metrics}, but their efficacy for explaining black-box models remains a topic of debate~\cite{bai2021attentions,wiegreffe-pinter-2019-attention}. SHAP~\cite{lundberg2017unified} uses Shapley values, while LIME~\cite{ribeiro2016should} derives explanation from a random local combination of records and their neighborhoods weighted by proximity.

\begin{figure}[t]
\setlength{\belowcaptionskip}{-20pt}
\centering
\includegraphics[width=0.48\textwidth]{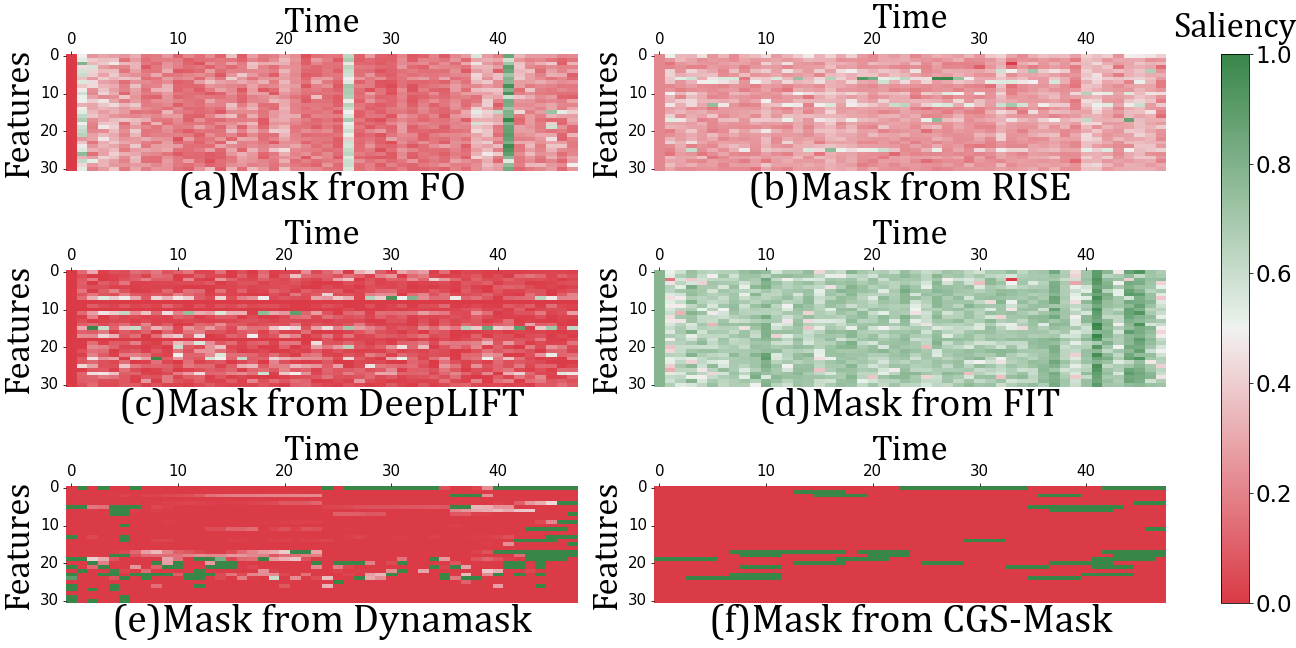} 
\caption{Six explanation masks are employed to analyze data from MIMIC-III, covering the 48 hours preceding a patient's death. In Fig.~\ref{fig1}(f), the green strips reveal the key features indicative of patient outcomes. Specifically, a decline in blood pressure, tachycardia of heart rate, and rapid breathing signal an impending risk of death, thereby enabling timely intervention by doctors. However, the other masks do not distinctly identify the periods and features contributing to this outcome, particularly as observed in Fig.~\ref{fig1}(a)$-$(d).}
\label{fig1}
\end{figure}

Some saliency methods have been implemented as explainable AI tools and worked well on images, text, and tabular data. The direct use of these saliency methods in interpreting time series predictions generally produces results that are hard for humans to comprehend. 
Existing tools are effective at scoring feature importance at a fine time scale, producing saliency maps with precise numeric values for each feature's saliency score over a time horizon, as shown in Fig.~\ref{fig1}(a)$-$(c). However, these fragmented landscapes may not provide meaningful and engaging explanations to users.

Researchers have started designing saliency methods considering the time-sensitive nature of time-series applications. FIT~\cite{tonekaboni2020went} quantifies the importance of observations over time by evaluating their contribution to the prediction output. However, the feature importance is still scored by numeric values and thus leading to the output as shown in Fig.~\ref{fig1}(d), somewhat less intuitive for users. Dynamask~\cite{crabbe2021} considers the time dependency of time series data in design and perturbs the input with a dynamic combination of the adjacent values of any given data in the input. It avoids rapid changes in feature saliency score over time, as shown in Fig.~\ref{fig1}(e), and makes it much more user-friendly than the previous methods. These methods are aware of the continuity of feature importance over time and try to keep it in the explanation~\cite{fong2017interpretable}. Unfortunately, they all score feature importance by numeric values or require other internal network states from the model like gradients. For example, the gradient-based methods are limited to models that provide internal information or have a specific architecture to optimize the final out. Their practicality in real-world applications is constrained~\cite{Petsiuk2018rise}.

To address the above challenges, we propose \emph{\textbf{C}ellular \textbf{G}enetic \textbf{S}trip \textbf{Mask}} (CGS-Mask), a saliency method using the perturbation mechanism to explain \textcolor{black}{multivariate} time series predictions. CGS-Mask utilizes a strip mask approach, treating consecutive time steps as a cohesive entity to evaluate their impact on the final prediction. It scores feature importance as a binary value for better explanation as shown in Fig.~\ref{fig1}(f). We also develop an enhanced cellular genetic algorithm to obtain the final explanation in a reasonable time and make CGS-Mask model agnostic. Our paper has the below contributions:

\begin{enumerate}

\item We consider consecutive time steps as a mask strip to incorporate temporal continuity of features in saliency scoring for AI models used in time series predictions. The resulting feature saliency score is binarized, facilitating clear and intuitive result interpretation.

\item {CGS-Mask is a strictly model-agnostic, self-adaptive metaheuristic approach that accurately identifies salient features in time series applications, relying solely on the input and output without requiring knowledge of AI models' inner workings.}

\item We compare our method to eight state-of-the-art methods on both synthetic and real-world data sets. Results indicate that our method can identify salient features consistently and provide the best intuitive explanations, as demonstrated by the user study.

\end{enumerate}

\section{Preliminaries}
Our mission is to explain the prediction $\mathbf{Y}=f(\mathbf{X})$ of a pre-trained black box model $f$ for the upstream prediction task. In our approach, the input $\mathbf{X}$ is a time series data $\mathbf{X}=(x_{d,t})_{(d,t) \in  [1:D] \times [1:T]} \in \mathbb{R}^{D \times T}$, where $D$ is the number of features with $T$ observations over time, $x_{d,t} \in \mathbb{R}^{D \times T}$ is the data point of feature $d \in [1:D]$ at time $t \in [1:T]$. $\mathbf{Y}=(y_{d_\mathcal{Y},t_\mathcal{Y}})_{d_\mathcal{Y},t_\mathcal{Y} \in [1:D_{\mathcal{Y}}][1:T_{\mathcal{Y}}]} \in \mathbb{R}^{D_{\mathcal{Y}} \times T_{\mathcal{Y}}}$ is the prediction from the black box model $f:\mathbb{R}^{D \times T} \rightarrow\mathbb{R}^{D_{\mathcal{Y}} \times T_{\mathcal{Y}}}$, where $D_{\mathcal{Y}} \times T_{\mathcal{Y}}$ is the size of $\mathbf{Y}$. For instance, in a regression task, where the output is a single value, we have $D_{\mathcal{Y}}=T_{\mathcal{Y}}=1$ and $\mathbf{Y}=y_{1,1}$ represents a real number. 

Our goal is to identify the parts of the input $\mathbf{X}$ that have a significant impact on the predictions made by a pre-trained black box model $f$. Drawing upon prior studies~\cite{fong2017interpretable,fong2019understanding} and established notations ~\cite{crabbe2021,tonekaboni2020went,ismail2020benchmarking}, we use the mask $\mathbf{M}$ to identify the salient features in the input $\mathbf{X}$. The mask is defined as follows: 

\begin{definition} [Mask $\mathbf{M}$] For a given input $\mathbf{X}=(x_{d,t})_{(d,t)\in [1:D]\times[1:T]} \in \mathbb{R}^{D \times T}$ and the black box model $f:\mathbb{R}^{D \times T}\rightarrow\mathbb{R}^{D_{\mathcal{Y}} \times T_{\mathcal{Y}}}$, the corresponding mask is $\mathbf{M}=(m_{d,t})_{(d,t)\in [1:D]\times[1:T]}\in [0,1]^{D\times T}$. The mask $\mathbf{M}$ has the same dimensions as the input $\mathbf{X}$. Each element $m_{d,t}$ in the matrix represents the importance of the feature $d$ at time $t$ to produce the prediction $\mathbf{Y} = f(\mathbf{X})$. In our work, $m_{d,t}$ close to 1 indicates that $x_{d,t}$ is salient, while $m_{d,t}$ close to 0 indicates the opposite.
\label{mask}
\end{definition}

\begin{figure*}[t]
\setlength{\belowcaptionskip}{-10pt}
\centering
\includegraphics[width=0.90\textwidth]{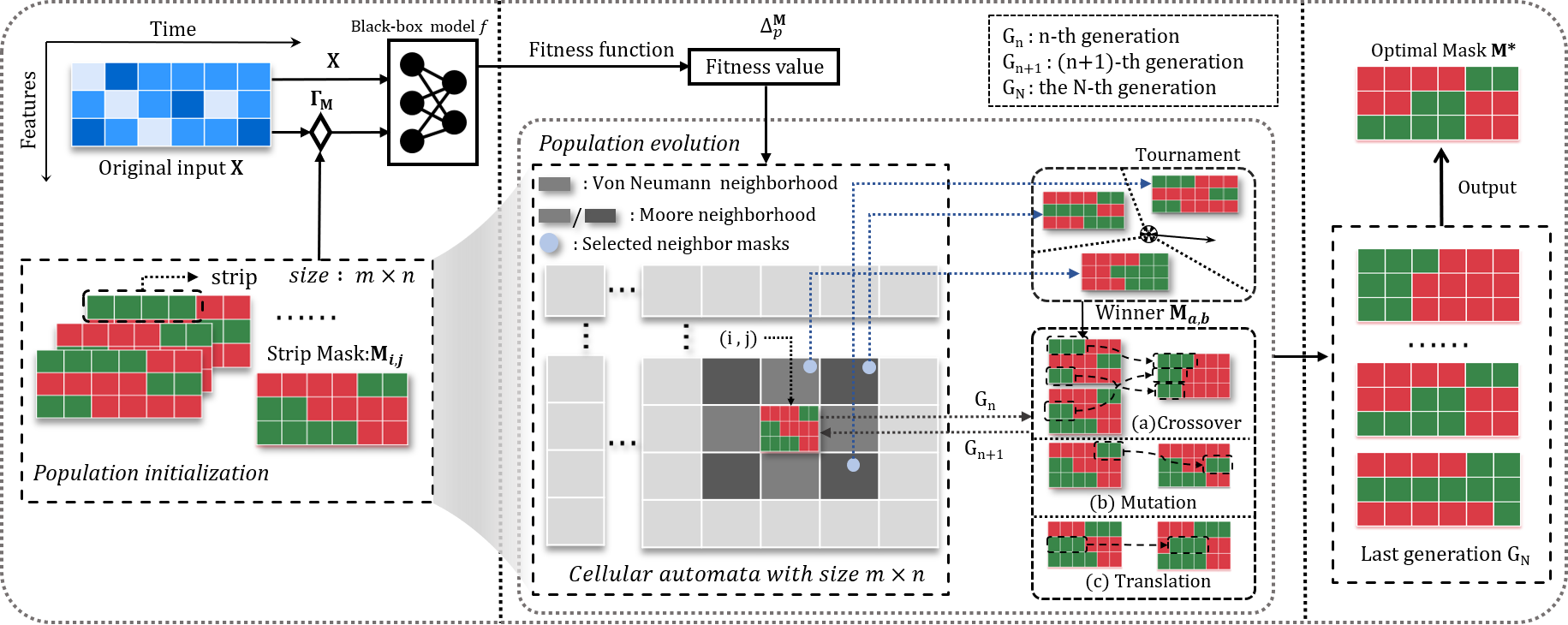} 
\caption{
\textcolor{black}{
The overall framework of CGS-Mask.}}
\label{overall}
\end{figure*}

As a perturbation-based method, it is necessary to determine how the mask $\mathbf{M}$ perturbs the input $\mathbf{X}$. Different applications may require different methods to explore the input data and produce significant perturbations. Inspired by the perturbation operator used in the context of image classification~\cite{fong2019understanding}, we propose our definition as below:

\begin{definition} [Perturbation Operator $\Gamma_\mathbf{M}$] The perturbation operator $\Gamma_\mathbf{M}:\mathbb{R}^{D \times T }\rightarrow\mathbb{R}^{D \times T}$ uses $\mathbf{M}=(m_{d,t})_{(d,t)\in[1:D]\times[1:T]}\in [0,1]^{D\times T}$ to perturb $\mathbf{X}=(x_{d,t})_{(d,t) \in [1:D]\times[1:T]} \in \mathbb{R}^{D \times T} $ and get the perturbed version $\widehat{\mathbf{X}}=\Gamma_{\mathbf{M}}(\mathbf{X})$, where $\widehat{\mathbf{X}} = (\hat{x}_{d,t})_{(d,t)\in [1:D]\times[1:T]} \in \mathbb{R}^{D \times T}$. $\hat{x}_{d,t}\in \widehat{\mathbf{X}}$ can be calculated as follows:
\begin{equation}\label{perturb}
[\Gamma_{\mathbf{M}}(\mathbf{X})]_{d,t} = \hat{x}_{d,t}  =  m_{d,t} \times p_{d,t} + (1-m_{d,t})\times x_{d,t}
\end{equation}
\label{def:2}
\end{definition}

Here, $p_{d,t}$ is the perturbation value used to perturb $x_{d,t}$, $x_{d,t}$ is the value of feature $d$ at the time $t$, and $m_{d,t}$ is the saliency value for $x_{d,t}$. For time series data, $p_{d,t}$ can also be set to some specific values that consider the time dependency of input data~\cite{crabbe2021}. $p_{d,t}$ can thus be calculated in one of the following forms:

\begin{equation}
p_{d,t} = \frac{1}{T}\sum_{t'=1}^{T}x_{d,t'}
\label{perturbway1}
\end{equation}

\begin{equation}
p_{d,t} = \frac{1}{2\mathbf{K}}\big(\sum_{t'=t-\mathbf{K}}^{t-1}x_{d,t'}+\sum_{t'=t+1}^{t+\mathbf{K}}x_{d,t'}\big)
\label{perturbway2}
\end{equation}

Equation (\ref{perturbway1}) represents the perturbation value as the average value of feature $d$. Equation (\ref{perturbway2}) indicates that the perturbation value is the average value of feature $d$ within the time range of $[t-\mathbf{K},t-1]\cup [t+1,t+\mathbf{K}]$. Other suitable forms for specific applications are also acceptable.

Ideally, the objective mask identifies a subset of salient features in the input $\mathbf{X}$ that explain the prediction. Our work will assign a high value to $m_{d,t}$ if $x_{d,t}$ impacts the prediction $\mathbf{Y}$. Specifically, we have the following proposition:

\begin{proposition}
\emph{If $x_{d,t}\in \mathbf{X}$ is salient for $f(\mathbf{X})$, $p_{d,t}$ is not salient for $f(\mathbf{X})$, $m_{d,t}\in \mathbf{M}$ and $m_{d,t}>0$, then $||f(\mathbf{X}) -f(\mathbf{\Gamma_{\mathbf{M}}(\mathbf{X}))}|| > 0$.}
\label{proposition}
\end{proposition}

\begin{proof}
The proof of proposition~\ref{proposition} is quite straightforward. Say $x_{d,t}$ is salient for $f(\mathbf{X})$. After the perturbation by Equation (\ref{perturb}), $x_{d,t}$ will change to $\hat{x}_{d,t}$ with the perturbed value $p_{d,t}$, which is not as the same salient as $x_{d,t}$ for $f(\mathbf{X})$. It means that $f(\hat{\mathbf{X}})$ differs from $f(\mathbf{X})$, so we can conclude that $||f(\mathbf{X}) -f(\mathbf{\Gamma_{\mathbf{M}}(\mathbf{X}))}|| > 0$
\end{proof}

According to the above proposition, we remark that the more salient area a mask $\mathbf{M}$ can represent, the bigger the difference between $f(\mathbf{X})$ and $f(\mathbf{\Gamma_{\mathbf{M}}(\mathbf{X}))}$. Therefore, comparing the predictions from both unperturbed and perturbed inputs is helpful. We call the difference of the predictions with perturbed and unperturbed values as \emph{perturbation error $\Delta_{p}^{\mathbf{M}}$} and give a definition as below:

\begin{definition} [Perturbation Error $\Delta_{p}^{\mathbf{M}}$] The perturbation error $\Delta_{p}^{\mathbf{M}}$ evaluates the performance of each potential mask. It can be calculated by $\mathbf{X}$ and its perturbed version $\widehat{\mathbf{X}}=\Gamma_{\mathbf{M}}(\mathbf{X})$. 
The specific calculation of
$\Delta_{p}^{\mathbf{M}}$ varies depending on the prediction task. In the case of a regression task, where $\mathbf{Y} \in \mathbb{R}^{D_{\mathcal{Y}}\times T_{\mathcal{Y}}}$, we use the squared error between the original $\mathbf{X}$ and the perturbed prediction $\mathbf{\hat{X}}$ to derive $\Delta_{p}^{\mathbf{M}}$ as follows:
\begin{equation}
 \Delta_{p}^{\mathbf{M}} = \sum_{d_y=1}^{D_{\mathcal{Y}}}\sum_{t_y=1}^{T_{\mathcal{Y}}} \big([f(\mathbf{X})]_{d_y,t_y}-[f(\Gamma_\mathbf{M}(\mathbf{X}))]_{d_y,t_y}\big)^{2} 
 \label{deltap1}
\end{equation}
To evaluate the impact of the perturbation of input $\mathbf{X}$ on the prediction in classification tasks, we use the cross entropy to derive $\Delta_{p}^{\mathbf{M}}$. Here, $f(\mathbf{X})$ represents the predicted probability by the classifier for $\mathbf{X}$. Then $\Delta_{p}^{\mathbf{M}}$ can be defined as follows:
\begin{equation}
         \Delta_{p}^{\mathbf{M}} =-\sum_{t_{\mathcal{Y}}=1}^{T_{\mathcal{Y}}}\sum_{c=1}^{D_{\mathcal{Y}}}[f(\mathbf{X})]_{t_{\mathcal{Y}},c}\log[f(\Gamma_\mathbf{M}(\mathbf{X}))]_{t_{\mathcal{Y}},c}
\label{deltap2}
\end{equation}
\label{def:3}
\end{definition}

To obtain the optimal mask $\mathbf{M^{*}}$, we aim to maximize $\Delta_{p}^{\mathbf{M}}$. The objective function is thus defined as:
\begin{equation}
\mathbf{M^{*}} = \mathop{\arg\max} \limits_{\mathbf{M} \in [0,1]^{D \times T}}\Delta_{p}^{\mathbf{M}}
\label{object}
\end{equation}

\section{Methodology}

To highlight the continuity of feature importance in the time domain, we first design a strip mask to incorporate the temporal dependency in the results. Then, CGS-Mask uses a modified cellular genetic algorithm to optimize the masks to get the optimal one with the highest $\Delta_{p}^{\mathbf{M}}$.

\subsection{Strip Mask}

As discussed, the mask built by Definition~\ref{mask} captures the feature saliency in the finest time scale but may result in a less coherent narrative. To enhance interpretability, we need a mask that provides a more intuitive explanation of the prediction and is easily understandable at both fine and coarse time scales. Thus, we propose designing a mask that spans consecutive time steps and allows for scaling when necessary. Additionally, relying on numerical ranks to convey the saliency score might not efficiently convey the information to users, as the subtle differences between adjacent numbers can be difficult to interpret adequately. To tackle these, we introduce the concept of 
\textbf{strip} to build our mask. For the strip mask $\mathbf{M}$, the strips represent a subset of input $\mathbf{X}$ that significantly impacts the prediction result of a pre-trained black-box model $f$. The strips are arranged in consecutive time steps, and the values in the mask are binary to enhance its interpretability. We define the strip mask as follows:

\begin{definition} [Strip Mask] For a given input $\mathbf{X}=(x_{d,t})_{(d,t)\in [1:D]\times[1:T]} \in \mathbb{R}^{D \times T}$, the corresponding strip mask $\mathbf{M} = (m_{d,t}) \in \{0, 1\}^{D \times T}$ contains binary values. Each strip mask $\mathbf{M}$ has a strip set $\mathbf{S} = \{S_{n}\ |\ n\in[1,N]\}$, where $N\in \mathbb{N}$ is the number of strips in the mask. To generate a strip $S_{n} \in \mathbf{S}$, we randomly initialize its starting position $b \in [1,T]$, the feature index $d \in [1,D]$, and the strip size $l$. For each point $m_{d,t}$ in the strip, we set $m_{d,t} = 1\ for\ all\ t \in [b,b+l]$.
\end{definition}

Please note that in our design, the strip size $l$ can be set to one, effectively resulting in a traditional time point-based mask. Therefore, in the rest of our work, the terms mask and strip mask are used interchangeably, as traditional masks can be seen as a specific form of our strip mask.

\begin{algorithm}[t]
    \caption{CGS-Mask generation}
    \label{alg:algorithm}
    \textbf{Input}: The black-box prediction model $f$, two-dimensional cellular automata (CA) in size $m\times n$, \textcolor{black}{N rounds of generations}, the probability of crossover operator $P_{c}$, the probability of mutation operator $P_{m}$, and the probability of translation operator $P_{t}$.\\
    \textbf{Output}: The optimal mask $\mathbf{M^{*}}$
    \begin{algorithmic}[1] 
        \STATE Initialize the population of strip masks with size $m\times n$ and put each strip mask into CA, n=0. Ensure $P_{c},P_{m},P_{t}\in [0,1], P_{c}+P_{m}+P_{t}\leq 1$\\
        \WHILE{\textcolor{black}{$n \le N$} }
        \STATE \textcolor{black}{Calculate the fitness value of each mask using the fitness function of Equation (\ref{deltap1}) or (\ref{deltap2}) ;}
        \FOR{each mask $\mathbf{M_{i,j}}$}
        \STATE \textcolor{black}{generate} a random number $r\in [0,1]$;
        \IF {$0\le r \le P_{c}$}
        \STATE $\mathbf{M_{S}}=Select({\mathbf{M_{E}}})$;
        \STATE $\mathbf{M_{a,b}}=Tournament({\mathbf{M_{S}}})$;
        \STATE $Crossover(\mathbf{M_{i,j}, M_{a,b}})$;
        \ELSIF{$P_{c} < r \le P_{c}+P_{m}$}
        \STATE $Mutation(\mathbf{M_{i,j}})$;
        \ELSIF{$P_{c}+P_{m} < r \le P_{c}+P_{m}+P_{t}$}
        \STATE $Translation(\mathbf{M_{i,j}})$;
        \ENDIF
        \ENDFOR
        \STATE n++;
        \ENDWHILE
        \STATE $\mathbf{M^{*}} = Optimal(\{\mathbf{M_{i,j}}|i\in [1,m],j\in [1,n]\}) $;
        \STATE \textbf{return} $\mathbf{M^{*}}$
    \end{algorithmic}
\end{algorithm}

\subsection{Strip Mask Optimization}

\textcolor{black}{Our objective is to generate a mask that can identify a subset of the original input to explain the prediction. This task can be seen as a subset selection problem, which is generally an NP-hard problem~\cite{qian2015subset,natarajan1995sparse}. Metaheuristic algorithms are commonly employed to tackle such subset search problems~\cite{bian2020efficient,rostami2021review}. In contrast to existing models that rely on the internal state of the model, such as gradient-based methods \cite{zhang2018top,simonyan2013deep}, metaheuristic algorithms do not necessitate any internal knowledge of the model, resulting in a model-agnostic interpretation algorithm.}

\textcolor{black}{We present a model-agnostic metaheuristic approach called the self-adaptive cellular genetic algorithm (SA-CGA) to solve the strip search problem. With its help we can optimize the strip as the basic unit of the mask without relying on the internal state of models.} SA-CGA seamlessly combines the genetic algorithm and cellular automata~\cite{li2018cellular}. By treating a single cell as an individual in the population, we can implement the operators of the genetic algorithm alongside the neighbor rule utilized in the cellular automata model. This design enables individuals to gather information from local surroundings, learn from neighboring nodes in a complex system, adapt actively to the environment, improve their fitness value, and ultimately make better decisions.

Based on SA-CGA, we treat each strip mask as an individual and place them into a two-dimensional cellular automata.
We design genetic operators to enable information exchange between neighboring masks in the cellular automata. Our algorithm, called CGS-Mask, is presented in Algorithm~\ref{alg:algorithm}. Initially, we create a population of strip masks and map them into the cellular automata. Then, we optimize each mask to evolve into the next generation using genetic operators such as the crossover, mutation, and translation operators. The optimal mask will be found after N rounds of generations and we select the mask with highest fitness value as the optimal mask.

\subsubsection{Population Initialization}

CGS-Mask randomly initializes the mask set $\{\mathbf{M_{i,j}} = (m_{d,t}^{(i,j)})\in \{0, 1\}^{D \times T}|i\in [1,m],j\in [1,n]\}$. As illustrated in Fig.~\ref{overall}, we map $\mathbf{M_{i,j}}$ to the (i, j) position of the cellular automata with the size of $m\times n$. For each mask, there is a set of neighbor masks $\mathbf{M_{\mathbf{E}}}$. The neighbors of $\mathbf{M_{i,j}}$ can be set in various ways, e.g., the von Neumann neighbor or Moore neighbor shown in Fig.~\ref{overall}.

\begin{table*}[t]
   
    \renewcommand{\arraystretch}{1}
    \setlength{\belowcaptionskip}{-10pt}
    \centering
    \resizebox{1.95\columnwidth}{!}{
    \begin{tabular}{ccccccccc}
        \specialrule{1pt}{0pt}{0pt}
        &  & \multicolumn{2}{c}{rare feature} &  &  & \multicolumn{2}{c}{rare time} &     \\ \hline
        Methods & $AUP\uparrow$ & $AUR\uparrow$ & $\mathcal{D}_{\mathbf{M}}\downarrow$ & $\mathcal{E}_{\mathbf{M}}\downarrow$ & $AUP\uparrow$ & $AUR\uparrow$ & $\mathcal{D}_{\mathbf{M}}\downarrow$ & $\mathcal{E}_{\mathbf{M}}\downarrow$     \\ \hline
        FO & $\mathbf{1.00\pm0.00}$ & $0.14\pm0.02$ & $128.70\pm2.33$ & $11.14\pm0.69$ & $\mathbf{1.00\pm0.00}$ & $0.13\pm0.02$ & $119.45\pm2.48$ & $48.34\pm1.12$  \\ 
        FP & $\mathbf{1.00\pm0.00}$ & $0.16\pm0.03$ & $121.02\pm2.72$ & $14.34\pm0.78$ & $\mathbf{1.00\pm0.00}$ & $0.23\pm0.02$ & $125.76\pm2.76$ & $54.57\pm0.97$  \\ 
        IG & $0.99\pm0.01$ & $0.14\pm0.03$ & $137.98\pm1.29$ & $11.95\pm0.56$ & $0.99\pm0.01$ & $0.17\pm0.03$ & $128.39\pm1.93$ & $48.94\pm0.59$  \\ 
        SVS & $0.99\pm0.01$ & $0.18\pm0.04$ & $141.93\pm2.08$ & $11.18\pm0.63$ & $\mathbf{1.00\pm0.00}$ & $0.19\pm0.03$ & $122.31\pm2.31$ & $48.73\pm0.14$  \\ 
        FIT & $0.98\pm0.01$ & $0.27\pm0.02$ & $138.23\pm2.43$ & $13.76\pm0.52$ & $0.99\pm0.01$ & $0.39\pm0.03$ & $127.45\pm3.02$ & $51.56\pm1.93$  \\ 
        RISE & $0.97\pm0.02$ & $0.35\pm0.02$ & $145.24\pm2.97$ & $14.44\pm0.67$ & $0.96\pm0.01$ & $0.33\pm0.02$ & $129.49\pm1.93$ & $47.44\pm1.45$  \\ 
        DeepLIFT & $0.99\pm0.00$ & $0.33\pm0.03$ & $129.75\pm1.96$ & $12.58\pm0.33$ & $0.99\pm0.01$ & $0.42\pm0.03$ & $131.24\pm2.57$ & $41.53\pm1.89$  \\ 
        Dynamask & $0.99\pm0.01$ & $0.58\pm0.03$ & $82.73\pm1.22$ & $0.83\pm0.02$ & $0.99\pm0.01$ & $0.63\pm0.04$ & $95.45\pm1.19$ & $7.13\pm0.37$  \\ \hline
        w/o Strip & $0.99\pm0.01$ & $0.75\pm0.02$ & $27.29\pm0.92$ & $\mathbf{0.00\pm0.00}$ & $0.98\pm0.01$ & $0.76\pm0.02$ & $79.23\pm1.37$ & $\mathbf{0.00\pm0.00}$  \\ 
        w/o Cell & $0.99\pm0.01$ & $0.73\pm0.02$ & $14.71\pm0.08$ & $\mathbf{0.00\pm0.00}$ & $0.99\pm0.01$ & $0.75\pm0.02$ & $37.83\pm0.89$ & $\mathbf{0.00\pm0.00}$  \\ 
        w/o Trans & $0.98\pm0.01$ & $0.78\pm0.02$ & $15.27\pm0.97$ & $\mathbf{0.00\pm0.00}$ & $0.99\pm0.01$ & $0.78\pm0.03$ & $36.52\pm0.97$ & $\mathbf{0.00\pm0.00}$  \\ 
        CGS-Mask & $\mathbf{1.00\pm0.00}$ & $\mathbf{0.81\pm0.02}$ & $\mathbf{14.34\pm0.93}$ & $\mathbf{0.00\pm0.00}$ & $0.99\pm0.01$ & $\mathbf{0.82\pm0.04}$ & $\mathbf{34.78\pm1.15}$ & $\mathbf{0.00\pm0.00}$  \\ \hline

        &  & \multicolumn{2}{c}{mixture} &  &  & \multicolumn{2}{c}{random} &     \\ \hline
FO & $\mathbf{1.00\pm0.00}$ & $0.18\pm0.03$ & $226.78\pm8.64$ & $55.24\pm2.93$ & $0.99\pm0.00$ & $0.26\pm0.02$ & $222.95\pm4.35$ & $83.56\pm2.32$ \\ 
        FP & $0.99\pm0.00$ & $0.28\pm0.03$ & $223.81\pm5.79$ & $69.33\pm2.33$ & $0.99\pm0.00$ & $0.23\pm0.03$ & $221.42\pm7.32$ & $84.44\pm3.75$ \\ 
        IG & $0.99\pm0.00$ & $0.16\pm0.02$ & $234.52\pm8.93$ & $57.53\pm3.75$ & $\mathbf{1.00\pm0.00}$ & $0.17\pm0.02$ & $237.57\pm6.93$ & $101.16\pm3.72$ \\ 
        SVS & $0.99\pm0.01$ & $0.18\pm0.03$ & $229.40\pm7.64$ & $55.23\pm2.34$ & $\mathbf{1.00\pm0.00}$ & $0.26\pm0.02$ & $236.83\pm3.54$ & $83.58\pm3.38$ \\ 
        FIT & $0.99\pm0.01$ & $0.36\pm0.03$ & $225.73\pm5.39$ & $61.39\pm1.32$ & $0.99\pm0.00$ & $0.32\pm0.03$ & $231.73\pm3.76$ & $87.39\pm2.24$ \\ 
        RISE & $0.96\pm0.01$ & $0.33\pm0.02$ & $289.43\pm6.73$ & $52.44\pm3.44$ & $0.93\pm0.04$ & $0.35\pm0.03$ & $269.37\pm7.39$ & $91.44\pm3.92$ \\ 
        DeepLIFT & $0.98\pm0.01$ & $0.37\pm0.03$ & $231.47\pm3.72$ & $53.72\pm1.35$ & $0.99\pm0.00$ & $0.29\pm0.02$ & $246.52\pm4.39$ & $83.26\pm2.34$ \\ 
        Dynamask & $\mathbf{1.00\pm0.00}$ & $0.59\pm0.02$ & $122.96\pm2.93$ & $7.89\pm0.68$ & $0.99\pm0.00$ & $0.56\pm0.03$ & $145.83\pm3.65$ & $21.08\pm0.83$ \\ \hline
        w/o Strip & $0.98\pm0.01$ & $0.76\pm0.03$ & $104.72\pm2.34$ & $\mathbf{0.00\pm0.00}$ & $0.98\pm0.01$ & $0.75\pm0.03$ & $142.67\pm3.74$ & $\mathbf{0.00\pm0.00}$ \\ 
        w/o Cell & $0.99\pm0.00$ & $0.74\pm0.02$ & $59.34\pm1.73$ & $\mathbf{0.00\pm0.00}$ & $0.96\pm0.01$ & $0.74\pm0.02$ & $71.89\pm3.26$ & $\mathbf{0.00\pm0.00}$ \\ 
        w/o Trans & $0.99\pm0.00$ & $0.77\pm0.03$ & $\mathbf{58.32\pm2.71}$ & $\mathbf{0.00\pm0.00}$ & $0.98\pm0.01$ & $0.77\pm0.03$ & $71.67\pm2.34$ & $\mathbf{0.00\pm0.00}$ \\ 
        CGS-Mask & $0.99\pm0.01$ & $\mathbf{0.82\pm0.03}$ & $58.45\pm2.33$ & $\mathbf{0.00\pm0.00}$ & $0.96\pm0.01$ & $\mathbf{0.79\pm0.03}$ & $\mathbf{69.73\pm1.03}$ & $\mathbf{0.00\pm0.00}$ \\ \hline
         \specialrule{1pt}{0pt}{0pt}
    \end{tabular}}
    \caption{Results on the synthetic data sets.} 
    \label{syntheticData}
\end{table*}

\subsubsection{Population Evolution}
To evaluate the performance of the strip mask, we need to calculate the fitness value for each strip mask $\mathbf{M_{i,j}}$. Our approach leverages the perturbation error $\Delta_{p}^{\mathbf{M}}$, as defined in Definition ~\ref{def:3}, to determine the fitness value of each mask. The fitness function, denoted by Equation (\ref{deltap1}) for regression tasks and Equation (\ref{deltap2}) for classification tasks, is employed to calculate the fitness value.
Next, we need to optimize the original population iteratively. The initial population is denoted as $\mathbf{G_{0}}$, and the $n$-th generation is $\mathbf{G_{n}}$. CGS-Mask aims to optimize $\mathbf{G_{n}}$ to obtain the next generation $\mathbf{G_{n+1}}$. For each mask $\mathbf{M} \in \mathbf{G_{n}}$, we use the crossover operator with probability $P_{c}$, the mutation operator with probability $P_{m}$, and the translation operator with probability $P_{t}$ to generate the new masks that move towards optimality. We denote this process as $\mathbf{M}\longrightarrow \mathbf{M'}$, where $\mathbf{M'}\in \mathbf{G_{n+1}}$.

\paragraph{Crossover Operator:} As mentioned, each mask $\mathbf{M_{i,j}}$ has the neighbor set $\mathbf{M_{E}}$. For example, the Moore neighbor set $\mathbf{M_{E}}$ for $\mathbf{M_{i,j}}$ includes eight neighbors as shown in Fig.\ref{overall}. Then, we select the mask set $\mathbf{M_{S}} \subseteq \mathbf{M_{E}}$ with the predefined probability $P_{c}$. In this case, the fitness value of the $\mathbf{M} \in \mathbf{M_{S}}$ should not be less than that of $\mathbf{M_{i,j}}$. \textcolor{black}{We implement a tournament on $\mathbf{M_{S}}$, and} 
based on the probability calculated from the fitness values of $\mathbf{M} \in \mathbf{M_{S}}$, we choose one of them, denoted as $\mathbf{M_{a,b}} \in \mathbf{M_{S}}$, to perform crossover with $\mathbf{M_{i,j}}$. The probability is defined as follows:

\begin{equation}
P_{a,b} = \frac{\Delta_{p}^{\mathbf{M_{a,b}}}}{\sum\limits_{\mathbf{M_{a',b'}}\in \mathbf{M_{S}}}^{}\Delta_{p}^{\mathbf{M_{a',b'}}}}
\label{pab}
\end{equation}
where $P_{a,b}$ represents the probability of $\mathbf{M_{a,b}}$ winning the \textcolor{black}{tournament} against other masks in $\mathbf{M_{S}}$ to undergo crossover with $\mathbf{M_{i,j}}$.

After obtaining $\mathbf{M_{i,j}}$ and $\mathbf{M_{a,b}}$, the crossover operator will combine them to generate a new mask. In CGS-Mask, the strip is the atomic genetic unit for the crossover operation. The strip offspring can inherit from either parent in the crossover operator. Assuming mask $\mathbf{M_{i,j}}$ has strip set $\mathbf{S_{i,j}}$ and mask $\mathbf{M_{a,b}}$ has $\mathbf{S_{a,b}}$, let $S_{i,j}^{\alpha} \in \mathbf{S_{i,j}}$ and $S_{a,b}^{\beta} \in \mathbf{S_{a,b}}$, where $\alpha,\beta \in [1,U]$ and $U$ is the number of strips for both $\mathbf{S_{i,j}}$ and $\mathbf{S_{a,b}}$. Mask $\mathbf{M_{o}}$ is the offspring of $\mathbf{M_{i,j}}$ and $\mathbf{M_{a,b}}$, and $\mathbf{M_{o}}$ has the strip set $\mathbf{S_{o}}=\bigcup_{\alpha,\beta=1}^{U}{Choose(S_{i,j}^{\alpha},S_{a,b}^{\beta})}$, where

\begin{equation}
P(Choose(S_{i,j}^{\alpha},S_{a,b}^{\beta})=S_{i,j}^{\alpha}) = \frac{\Delta_{p}^{\mathbf{M_{i,j}}}}{\Delta_{p}^{\mathbf{M_{i,j}}}+\Delta_{p}^{\mathbf{M_{a,b}}}}
\label{pchoose}
\end{equation}

The acts of the crossover operator are shown in lines 7-10 in Algorithm~\ref{alg:algorithm} and \textcolor{black}{(a) Crossover in Fig.~\ref{overall}}.

\paragraph{Mutation Operator:} 
The mutation operator in CGS-Mask fosters genetic diversity and prevents masks from converging to local minima. During mutation, a strip in the mask may be replaced by a new strip with a certain probability. 
Let $\mathbf{M}$ represent the mask with strip set $\mathbf{S}$. $S_{i}\in\mathbf{S}$ denotes the strip that will be deleted in the offspring of $\mathbf{M}$, while $S_{j}$ represents the strip that will be added.
Hence, the strip set of the offspring $\mathbf{M_{o}}$ can be calculated as follows:
\begin{equation}
\mathbf{S_{o}} = (\mathbf{S}-\{S_{i}\})\cup \{S_{j}\}
\label{mutation}
\end{equation}

The process of the mutation operator is given in line 11 in Algorithms~\ref{alg:algorithm} and \textcolor{black}{(b) Mutation in Fig.~\ref{overall}}.

\paragraph{Translation Operator:}
In a common scenario, we may encounter a strip \textcolor{black}{$S_{i}$} that shares the same features and length as the true salient strip $S_{t}$ but is misaligned by a few time steps. Simply mutating $S_{i}$ could alter its features and reduce the mask's fitness value. To address this, we propose the translation operator. The translation operator adjusts the position offset of strips on the timeline in CGS-Mask. Given a mask $\mathbf{M}$ with strip set $\mathbf{S}$, we select a strip $S_{i}\in\mathbf{S}$ to be translated in the offspring. By shifting the original beginning $t_{i}$ of $S_{i}$ by a translation distance $t_m$, the new beginning $t_n$ in the offspring is determined as $t_n=t_{i}\pm t_m$. The translation operator is indicated in line 13 of Algorithm~\ref{alg:algorithm} and (c) Translation in Fig.~\ref{overall}. 

After N generations, the last generation $\mathbf{G_{N}}$ is obtained. The mask with the highest fitness value is selected as the output mask $\mathbf{M^{*}}$.

\begin{table*}[t]
    \tabcolsep=3pt
    \renewcommand{\arraystretch}{1}
    \centering
    \resizebox{2\columnwidth}{!}{
    \begin{tabular}{ccccccccccccc}
        \specialrule{1pt}{0pt}{0pt}
                      &     & MIMIC-III &       &   &   LSST     &    &    &   NATOPS &   &  & AE&  \\
        Methods       & $\Delta_{p}^{\mathbf{M}}\uparrow$&$\mathcal{D}_{\mathbf{M}}\downarrow$ & $\mathcal{E}_{\mathbf{M}}\downarrow$& $\Delta_{p}^{\mathbf{M}}\uparrow$&$\mathcal{D}_{\mathbf{M}}\downarrow$ & $\mathcal{E}_{\mathbf{M}}\downarrow$& $\Delta_{p}^{\mathbf{M}}\uparrow$&$\mathcal{D}_{\mathbf{M}}\downarrow$ & $\mathcal{E}_{\mathbf{M}}\downarrow$&
        $\Delta_{p}^{\mathbf{M}}\uparrow$&$\mathcal{D}_{\mathbf{M}}\downarrow$ & $\mathcal{E}_{\mathbf{M}}\downarrow$\\
        \specialrule{1pt}{0pt}{0pt}
            FO        & $0.035\pm0.002$ & $118.63\pm2.34$  & $62.35\pm2.07$ & $1.21\pm0.13$  & $29.75\pm1.23$  & $9.24\pm1.25$ & $1.37\pm0.05$ & $110.77\pm2.34$  & $69.72\pm2.01$ & $106.01\pm3.57$ & $327,65\pm5.31$  & $137.62\pm2.03$ \\
           FP         & $0.034\pm0.002$ & $115.73\pm2.31$  & $77.28\pm2.02$ & $1.33\pm0.11$  & $28.30\pm1.23$  & $10.93\pm1.34$ & $1.63\pm0.04$ & $115.76\pm2.34$  & $63.84\pm1.29$ & $110.23\pm4.87$ & $329.04\pm4.45$  & $131.52\pm5.31$ \\
           IG         & $0.051\pm0.002$ & $128.34\pm2.27$  & $83.94\pm1.93$ & $1.47\pm0.12$  & $29.36\pm2.01$  & $10.96\pm0.97$ & $1.69\pm0.05$ & $114.39\pm3.87$  & $65.60\pm2.36$ & $107.97\pm2.73$ & $332.83\pm7.86$  & $119.74\pm2.37$ \\
           SVS        & $0.031\pm0.001$ & $132.38\pm2.32$  & $61.53\pm1.95$ & $1.29\pm0.11$  & $30.14\pm1.37$  & $10.02\pm0.83$ & $1.72\pm0.04$ & $119.34\pm3.75$  & $63.52\pm2.37$ & $109.09\pm4.93$ & $322.34\pm9.38$  & $121.03\pm3.94$ \\
           FIT        & $0.034\pm0.003$ & $122.53\pm2.97$  & $27.83\pm2.39$ & $1.51\pm0.12$  & $32.35\pm2.31$  & $9.53\pm0.72$ & $1.81\pm0.03$ & $129.38\pm2.71$  & $46.93\pm2.39$ & $112.77\pm4.93$ & $315.14\pm7.34$  & $93.63\pm4.32$  \\
           RISE       & $0.033\pm0.002$ & $119.40\pm3.23$  & $43.44\pm2.34$ & $1.28\pm0.08$  & $29.48\pm1.75$  & $10.77\pm0.93$ & $1.71\pm0.04$ & $113.50\pm2.33$  & $69.77\pm 3.71$ & $106.93\pm3.97$ & $329.49\pm8.34$  & $116.65\pm5.31$ \\
         DeepLIFT     & $0.035\pm0.002$ & $129.83\pm2.56$  & $37.65\pm1.38$ & $1.43\pm0.05$  & $27.39\pm1.92$  & $9.06\pm0.99$ & $1.74\pm0.04$ & $102.30\pm3.40$  & $63.75\pm2.39$ & $116.39\pm4.91$ & $315.37\pm7.84$  & $129.61\pm4.23$  \\
         Dynamask     & $0.072\pm0.003$ & $91.88\pm2.01$   & $0.77\pm0.06$ & $1.77\pm0.09$  & $17.45\pm1.36$  & $0.53\pm0.08$ & $1.89\pm0.06$ & $89.95\pm4.32$   & $0.69\pm0.07$ & $121.15\pm2.93$ & $217.41\pm6.42$  & $2.84\pm0.42$ \\
         \hline
         w/o Strip    & $0.082\pm0.004$ & $83.45\pm2.07$   & $\mathbf{0.00\pm0.00}$ & $1.91\pm0.03$  & $14.34\pm1.03$  & $\mathbf{0.00\pm0.00}$ & $2.02\pm0.08$ & $73.29\pm4.03$   & $\mathbf{0.00\pm0.00}$ & $124.33\pm3.71$ & $164.59\pm5.37$   & $\mathbf{0.00\pm0.00}$  \\
         w/o Cell     & $0.080\pm0.003$ & $46.75\pm1.49$   & $\mathbf{0.00\pm0.00}$ & $1.87\pm0.05$  & $10.64\pm0.93$  & $\mathbf{0.00\pm0.00}$ & $1.99\pm0.05$ & $39.42\pm3.70$   & $\mathbf{0.00\pm0.00}$ & $122.72\pm4.69$ & $99.59\pm4.31$   & $\mathbf{0.00\pm0.00}$  \\
         w/o Trans    & $0.082\pm0.003$ & $47.32\pm1.74$   & $\mathbf{0.00\pm0.00}$ & $1.89\pm0.04$  & $11.13\pm0.89$  & $\mathbf{0.00\pm0.00}$ & $2.09\pm0.07$ & $38.56\pm2.87$   & $\mathbf{0.00\pm0.00}$ & $127.56\pm3.54$ & $101.59\pm3.74$   & $\mathbf{0.00\pm0.00}$  \\
         CGS-Mask     & $\mathbf{0.084\pm0.004}$ & $\mathbf{46.45\pm1.93}$   & $\mathbf{0.00\pm0.00}$ & $\mathbf{1.92\pm0.05}$  & $\mathbf{10.34\pm0.85}$  & $\mathbf{0.00\pm0.00}$ & $\mathbf{2.15\pm0.09}$ &  $\mathbf{38.29\pm2.95}$  & $\mathbf{0.00\pm0.00}$ & $\mathbf{131.24\pm3.72}$ & $\mathbf{97.59\pm3.33}$   & $\mathbf{0.00\pm0.00}$  \\
         \specialrule{1pt}{0pt}{0pt}
    \end{tabular}}
    \caption{Results on the real-world data sets.}
    \label{real-wordData}
\end{table*}

\section{Experiments}
We conducted experiments to evaluate the performance of CGS-Mask and compared it with eight saliency methods on four \emph{synthetic data sets} and four \emph{real-world data sets}. We also conducted a pilot user study to gain insights into how different saliency methods help users understand the decision-making process of AI models.  More details about hardware resources, implementation, metrics, and results can be found in the supplementary material.

\subsection{Experiments with synthetic data sets}

\paragraph{Data sets and settings}
We conducted experiments using synthetic data sets, including the \emph{rare features} and \emph{rare time} data sets from~\cite{ismail2020benchmarking}. The \emph{rare features} data set consists of a small subset of salient features, while the \emph{rare time} data set contains a small subset of salient time points. We also created a \emph{mixture} data set that combines the \emph{rare features} and \emph{rare time} data sets, and a \emph{random} data set with randomly located salient input regions. In the experiments, the prediction model only relies on a known subset area $A=\{a_{p,q}|p\in [1,D],q\in [1,T]\}$ of the input $\mathbf{X}$ and $f$ is a function of $A$ that enables us to derive the ground truth. In our experiments, we set $D=T=50$. For the former two data sets $|A| = 125$, and for the latter two data sets $|A| = 250$. For simplicity, the perturbation value was set to zero. Since all tasks were regression, we evaluated the mask by Equation (\ref{deltap1}). The input $\mathbf{X}$ is generated through ARMA process~\cite{crabbe2021}. 

\paragraph{Metrics} 
As the ground truth explanations were known for the four synthetic data sets, we used the four below metrics to evaluate the performance. The \emph{area under the precision curve} (AUP) is the first metric to measure the proportion of identified features that are indeed salient. We used the \emph{area under the recall curve} (AUR) to measure the portion of the succesful identified salient features. Both metrics are higher the better\textcolor{black}{~\cite{crabbe2021}}. The \emph{Mask Discreteness} ($\mathcal{D}_{\mathbf{M}}$) assesses the temporal continuity of the mask~\cite{mercier2022time}. It quantifies the difference between $m_{t-1,d}$ and $m_{t,d}$ for each feature and time step. Lastly, the \emph{Mask Entropy} ($\mathcal{E}_{\mathbf{M}}$)~\cite{crabbe2021} is to measure the masks' sharpness and legibility. They are lower the better. 

\paragraph{Benchmarks}
We compared our approach CGS-Mask with Dynamask~\cite{crabbe2021}, DeepLIFT~\cite{shrikumar2017learning}, RISE~\cite{Petsiuk2018rise}, FIT~\cite{tonekaboni2020went}, Shapley Value Sampling (SVS)~\cite{molnar2020interpretable}, Feature Occlusion (FO), Feature Permutation (FP)~\cite{fisher2019all}, and Integrated Gradient (IG) in the tests.

\subsubsection{Results and Analysis}

The experiment results on four synthetic data sets are shown in Table~\ref{syntheticData}. We can see that all methods performed well on the AUP metric in determining the salient features. CGS-Mask significantly outperformed all baselines on AUR. The result suggests that CGS-Mask can determine more salient features over time due to its perturbation operator designed to incorporate temporal ordering and the time-sensitive nature of the input. CGS-Mask's $\mathcal{D}_{\mathbf{M}}$ is significantly lower than the baselines, ranging from 52.18\% to 82.67\%. The most consistent explanation over time comes from the strip mask's temporal continuity, which makes the results easy to understand. $\mathcal{E}_{\mathbf{M}}$ of CGS-Mask achieved zero means that its polarization reaches the theoretical minimum. This result is a direct outcome of the strip mask's binary value. In contrast, other methods are multi-valued based and exhibit diversity in feature saliency representation. The runtime of CGS-Mask is comparable to that of other methods, such as Dynamask. Find its specific runtime on each data set in the supplementary material.

\paragraph{Ablation study}

We performed an ablation study on CGS-Mask by a set of component variations to measure the effects of each component in CGS-Mask. Specifically, we considered the following variants of our method:
\begin{itemize}
    \item \textbf{w/o Strip}: The length of strips in CGS-Mask was set to 1, which means that the algorithm considers input saliency on an individual point basis. Therefore, the strip mask is transformed into a regular mask that operates on individual points.
    \item \textbf{w/o Cell}: The standard genetic algorithm replaced the cellular genetic algorithm. 
    \item \textbf{w/o Trans}: The translation operator was excluded from the population evolution module as it is not a standard component of the cellular genetic algorithm. 
\end{itemize}
The ablation study results in Table~\ref{syntheticData} demonstrate the impact of each component on the performance of CGS-Mask. Notably, the absence of the strip component (\textbf{w/o Strip}) significantly increased the value of $\mathcal{D}_{\mathbf{M}}$, suggesting that the strip design contributes to capturing the continuity of features in the explanation results. Additionally, the use of the cellular genetic algorithm and translation operator proved beneficial in recovering salient areas, as observed by the lower AUR performance in the cases of \textbf{w/o Cell} and \textbf{w/o Trans} compared to CGS-Mask.

\subsection{Results on real-world data}

\begin{figure*}
\setlength{\belowcaptionskip}{-15pt}
\begin{minipage}{0.3\textwidth}
\centering
    \includegraphics[width=0.95\columnwidth]{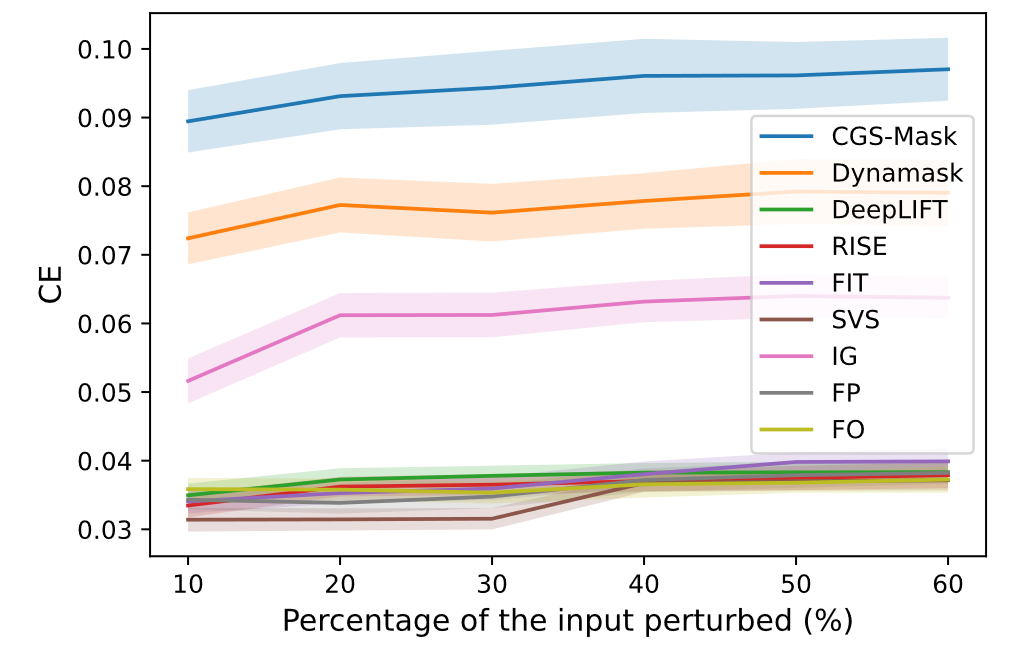}
    \caption{Cross entropy (CE) of saliency methods used in MIMIC-III }
    \label{ce}
\end{minipage}\hfill
\begin{minipage}{0.3\textwidth}
\centering
    \includegraphics[width=0.95\columnwidth]{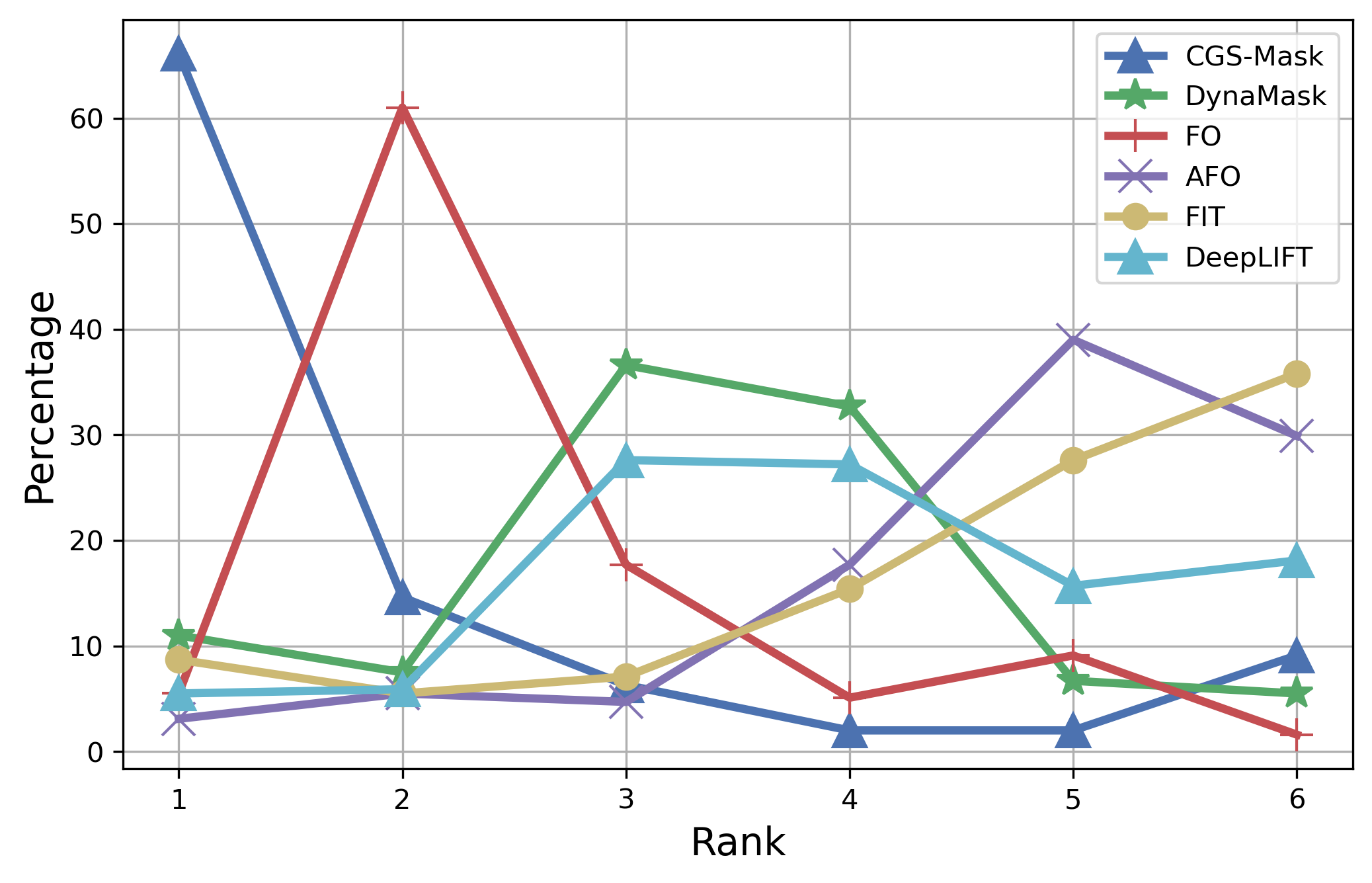} 
    \caption{The user preferences for different saliency masks}
    \label{cognitive}
\end{minipage}\hfill
\begin{minipage}{0.3\textwidth}
\centering
    \includegraphics[width=0.95\columnwidth]{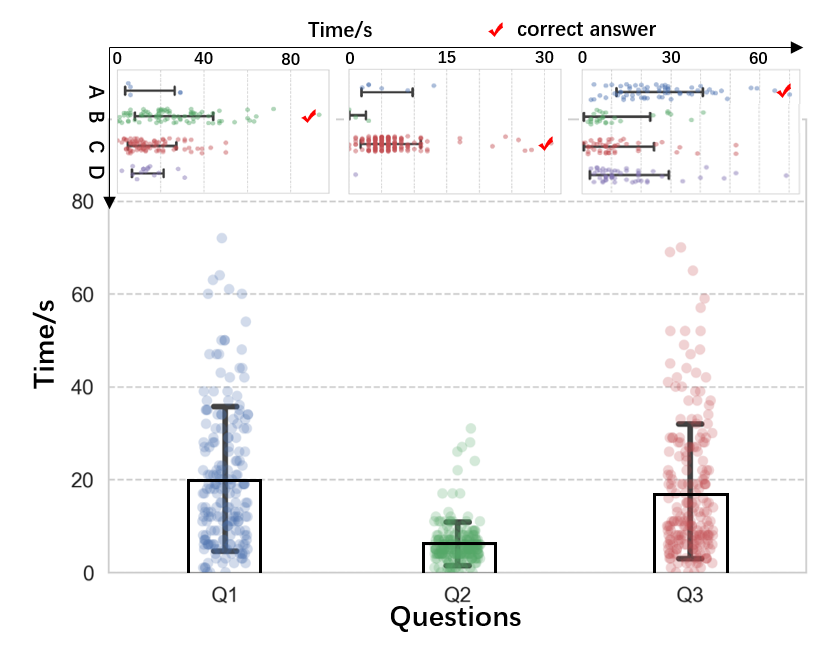} 
    \caption{User response time and selection results. }
    \label{cognitive2}\hfill
\end{minipage}

\end{figure*}

\paragraph{Data sets and settings} 
To further evaluate the performance of CGS-Mask, we conducted experiments on four individual real-world data sets from the healthcare, astronomy, sensors, and energy domains. 
The \emph{MIMIC-III} data set~\cite{johnson2016mimic} contains the health record of 40,000 ICU patients at the Beth Israel Deaconess Medical Center. We used it to predict patient survival rate 48 hours ahead with 31 patient features. It is a classification task with two classes, survival and death. For a fair comparison, the data selection, preprocessing, and model training were the same as ~\cite{tonekaboni2020went}. \emph{LSST} simulated astronomical time-series data in preparation for observations from the Large Synoptic Survey Telescope~\cite{PLAsTiCC-2018}. The prediction model needs to classify these data into 14 different classes. \emph{NATOPS} data set is generated by sensors on the hands, elbows, wrists, and thumbs for gesture recognition~\cite{ghouaiel2017continuous}. These data need to be classified into six different kinds of gestures. \emph{AE} was obtained from the Appliances Energy Prediction data set from the UCI repository to predict the total energy usage of a house~\cite{candanedo2017data}. It is a regression task, and the prediction model outputs a value representing the total energy usage in kWh. We used accuracy to evaluate the classification model's performance for \emph{MIMIC-III}, \emph{LSST}, and \emph{NATOPS}. In contrast, we used Mean Square Error (MSE) to evaluate the regression model's performance for \emph{AE}. 

\paragraph{Metrics}
Since we did not have the real-world data sets' ground truth explanation, we used the perturbation error $\Delta_{p}^{\mathbf{M}}$ to evaluate the performance of each result mask delivered by benchmarks. For \emph{MIMIC-III}, \emph{LSST}, and \emph{NATOPS}, $\Delta_{p}^{\mathbf{M}}$ is the cross entropy (CE) between $f(\mathbf{X})$ and $f(\hat{\mathbf{X}})$. For \emph{AE}, $\Delta_{p}^{\mathbf{M}}$ is the square error between $f(\mathbf{X})$ and $f(\hat{\mathbf{X}})$. We also used the $\mathcal{D}_{\mathbf{M}}$ and $\mathcal{E}_{\mathbf{M}}$ metrics as before.

\paragraph{Benchmarks}
We used the same benchmarks as before.

\subsubsection{Results and Analysis}

\textcolor{black}{For a fair comparison, we evaluated  $\Delta_{p}^{\mathbf{M}}$ of each method by considering saliency regions of 10\%, 20\%, 30\%, 40\%, 50\%, and 60\% of the input on the \emph{MIMIC-III} data set~\cite{crabbe2021,tonekaboni2020went}. The results are shown in Fig.~\ref{ce}. The largest value of CE ($\Delta_{p}^{\mathbf{M}}$) on \emph{MIMIC-III} data set indicates our ability to identify the subset of features with the greatest impact on the models. We also evaluated all methods on diverse real-world datasets, using a saliency region comprising 10\% of the input. The results are given in Table~\ref{real-wordData}. Once again, CGS-Mask outperformed other methods across all metrics. 
The strip mask showed the best continuity as it had the lowest $\mathcal{D}_{\mathbf{M}}$ value. Thanks to the binary nature of CGS-Mask, it attained a perfect $\mathcal{E}_{\mathbf{M}}$ of 0.} The runtime of CGS-Mask remains efficient, and additional details can be found in the supplementary material.

\paragraph{Ablation study}
We conducted an ablation study on real-world data, replicating the previous variations. The results in Table~\ref{real-wordData} confirm the influence of all components on the performance of CGS-Mask. The design of the strip mask directly impacted $\mathcal{D}_{\mathbf{M}}$, while the cellular genetic algorithm and translation operator improved the generation of effective masks for explanation purposes.

\subsection{Pilot User Study}
To assess the legibility of the generated masks, we conducted a survey involving \textcolor{black}{254} participants across different age groups (5 to 83 years old) and with varying levels of domain knowledge. 
The participants were asked to rank the saliency masks obtained from six different methods based on their effectiveness in helping the user understand the salient features and their temporal relevance. Our study showed that users preferred CGS-Mask for its intuitive comprehension of time series models. In Fig.~\ref{cognitive}, over 65\% of users ranked CGS-Mask as their top choice, while over 85\% rated it among their top three. In addition, we conducted a user study to measure the reaction time and accuracy of determining feature importance using three saliency masks (Q1, Q2, and Q3) with four features (A, B, C, and D) over 10 time steps. As shown in Fig~\ref{cognitive2}, CGS-Mask (Q2) exhibited a significantly shorter response time of $6.26\pm4.62$s and higher answer accuracy of 85.4\% compared to the numerical masks (Q1 and Q3), which had a response time of $19.22\pm15.04$s and an accuracy of 40.6\%. Further details of the pilot user study can be found in our Technical Appendix.

\section{Conclusion}

CGS-Mask is a model-agnostic saliency approach that explains time series prediction in an intuitive and user-friendly way. It performs well on synthesis and real-world data and outperforms state-of-the-art solutions. In the future, we aim to demonstrate its applicability in more time series applications, particularly in healthcare, for identifying salient features from medical records to reveal disease occurrence, development, and deterioration. While CGS-Mask's focus on local interpretability is intentional and aligns with our initial goals, we acknowledge that a global explanation is also valuable in some contexts. We are also actively working on enhancing CGS-Mask to incorporate global interpretability while maintaining its effectiveness at the local level.

\bibliography{aaai24}
\clearpage

\appendix

\twocolumn[
\begingroup
\fontsize{30pt}{60pt}\selectfont
\section{Technical Appendix} 
\vspace{40pt} 
\endgroup
]

\noindent \textbf{Supplementary Material: } This supplement file provides extra details on experiments. Please note that the earlier version of this paper has been accepted in another conference as a 2-page student abstract. This submission has extended the previous version extensively and fulfilled with more than 50\% new technical content.

\noindent\textbf{Ethics Declarations:} The author, who works with the children in the questionnaire survey, has completed a police check, possesses a Working with Children Certificate, and is covered by insurance.

\section{Further Details on experiments}
The hardware environment used for the experiments includes an \textit{AMD Ryzen 7 5800H CPU @ 3.20 GHz, Nvidia GeForce RTX(TM) 3050 GPU with 4GB memory}, and a \textit{1TB} hard disk. The software environment used includes \textit{Ubuntu 14.04} and \textit{Python 3.9.7}, along with essential software packages such as \textit{deap 1.3.1, Pytorch 1.11.0, and tsai 0.3.6}.

\subsection{Futher details on synthetic data experiments}
\subsubsection{Data generation}
For a fair comparison, each feature sequence was generated using an ARMA process, following the same procedure as in ~\cite{crabbe2021}. The calculation of each $x_{d,t}\in \mathbf{X}$ is as follows:

\begin{equation}
x_{d,t} = \beta_{1}\times x_{d,t-1} + \beta_{2}\times x_{d,t-2}+\beta_{3}\times x_{d,t-3}+\epsilon
\label{arma}
\end{equation}

Let $\beta_{1} = 0.25, \beta_{2} = 0.1, \beta_{3} = 0.05$ and $\epsilon\sim\mathcal{N}(0, 1)$. We generated one such sequence with $t\in[1 : 50]$ for each feature $i\in[1 : 50]$. In the \emph{rare feature} experiment, 5 features were selected as salient, each spanning 25 consecutive time steps. In the \emph{rare time} experiment, 5 time steps were selected as salient, corresponding to 25 selected features. In the \emph{mixture} experiment, 5 features and 5 time steps were selected as salient. In the \emph{random} experiment, 250 points were selected as salient. The examples of generated data used as the ground truth in our experiments are shown in Fig.~\ref{fig1}(a), Fig.~\ref{fig2}(a), Fig.~\ref{fig3}(a), and Fig.~\ref{fig4}(a).

\subsubsection{Metrics}

\paragraph{AUP, AUR} We used the definitions of AUP and AUR similar to the previous work and notions given in ~\cite{crabbe2021}. Assuming that $\mathbf{C} = (c_{d,t})_{(d,t)\in[1:D]\times[1:T]}\in \{0,1\}^{D\times T}$, $c_{d,t}$ is the true saliency value of $x_{d,t}\in \mathbf{X}$, where $\mathbf{X}\in \mathbb{R}^{D\times T}$ is the input. Specifically, $c_{d,t} = 1$ represents that $x_{d,t}$ is salient while $c_{d,t} = 0$ represents the opposite. For a mask $\mathbf{M} = (m_{d,t})_{(d,t)\in[1:D]\times[1:T]}\in [0,1]^{D\times T}$ adopted by a saliency method, if $\alpha \in (0,1)$ is the threshold for $m_{d,t}$ to express that $x_{d,t}$ is salient, we can convert $\mathbf{M}$ into $\hat{\mathbf{C}}(\alpha) = (\hat{c}_{d,t}(\alpha))_{(d,t)\in [1:D]\times [1:T]}$ where:
\begin{equation}
\hat{c}_{d,t}(\alpha)=\left\{
\begin{array}{cl}
1  &  if\quad m_{d,t} \geq \alpha \\
0  &  else \\
\end{array} \right.
\label{convert}
\end{equation}

For the set $A$, which contains the indices of truly salient elements from $\mathbf{C}$: $A=\{(d,t)\in [1:D]\times [1:T] | c_{d,t} = 1\}$\, and the set $\hat{A}(\alpha)$, which consists of the indices of elements in $\hat{C}(\alpha)$ with salient values of one, defined as  $\hat{A}(\alpha) = \{(d,t)\in [1:D]\times [1:T] | \hat{c}_{d,t}(\alpha) = 1\}$. We can compute the precision and recall curves as follows:

\begin{equation}
P(\alpha) = \frac{|A\cap \hat{A}(\alpha)|}{|\hat{A}(\alpha)|}
\label{precision}
\end{equation}

\begin{equation}
R(\alpha) = \frac{|A\cap \hat{A}(\alpha)|}{|A|}
\label{recall}
\end{equation}

\noindent where $P(\alpha)$ is the precision curve and $R(\alpha)$ is the recall curve. We computed the AUP and AUR scores as the areas under these curves:

\begin{equation}
AUP = \int_{0}^{1}P(\alpha)d\alpha 
\label{aup}
\end{equation}

\begin{equation}
AUR = \int_{0}^{1}R(\alpha)d\alpha
\label{aur}
\end{equation}

\paragraph{Mask Discreteness ($\mathcal{D}_{\mathbf{M}}$)}
$\mathcal{D}_{\mathbf{M}}$ is used to measure the continuity of masks. Firstly, we design the dispersion function as follows:
\begin{equation}
I(m_{d,t},m_{d,t+1})=\left\{
\begin{array}{cl}
1  &  if\quad |m_{d,t}-m_{d,t+1}|>\beta \\
0  &  else \\
\end{array} \right.
\label{descret}
\end{equation}
In our experiment, $\beta = 0.10$. Based on this, $\mathcal{D}_{\mathbf{M}}$ can be calculated as:
\begin{equation}
\mathcal{D}_{\mathbf{M}} =  \sum_{d=1}^{D}\sum_{t=1}^{T-1}I(m_{d,t},m_{d,t+1})
\label{dm}
\end{equation}
\paragraph{Mask Entropy ($\mathcal{E}_{\mathbf{M}}$)}
The mask entropy, as defined in~\cite{crabbe2021}, quantifies the sharpness and legibility of a mask $\mathbf{M}$. It can be computed as follows:

\begin{equation}
    \mathcal{E}_{\mathbf{M}} =  -\sum_{d=1}^{D}\sum_{t=1}^{T}(m_{d,t}\ln(m_{d,t})
    +(1-m_{d,t})\ln(1-m_{d,t}))
\label{em}
\end{equation}
Our objective is to ensure that our mask provides explanations with low mask entropy. Mask entropy is maximized when mask coefficients \textcolor{black}{$m_{d,t}$} are close to 0.5, indicating ambiguity in the saliency of the feature. Conversely, the entropy is minimized when the mask is perfectly binary, with coefficients either 0 or 1. CGS-Mask initializes the mask as binary, allowing it to achieve a theoretical minimum mask entropy of 0.00, ensuring clear and unambiguous explanations.
\subsubsection{Mask generation and optimization} 

To generate random strips $S\in \mathbf{S}$ for each experiment, where $\mathbf{S}$ represents the strip set of the strip mask $\mathbf{M}$, the length $l$ of the strip $S$ was randomly selected within specific ranges.

For the \emph{rare feature} experiments, $l$ ranged from 6 to 10, and the number of strips was set to 14. We found that increasing the number of strips did not significantly improve the fitness value. Similarly, in the \emph{rare time} experiments, $l$ ranged from 3 to 5, and the number of strips was set to 25. In the \emph{mixture} experiments, $l$ ranged from 2 to 8, and the number of strips was set to 45. In the \emph{random} experiments, $l$ ranged from 1 to 6, and the number of strips was set to 55.

To optimize the strip masks using CGS-Mask, we needed to set the parameters in Algorithm 1 described in the main body of the paper. Specifically, we set \textcolor{black}{ $N= 500$, $m=n=10$, $P_{c} = 0.75$, $P_{m}=0.1$, and $P_{t} = 0.1$} for each experiment on synthetic data. Additionally, we used \textcolor{black}{the Moore neighbor} in the experiments.

\subsubsection{\textcolor{black}{ Runtime}}
To optimize runtime, we implemented multi-threading to calculate the fitness value of the population in parallel for each generation in CGS-Mask. Utilizing 8 threads, all reported data points are an average of 500 executions. The average time to find the final explanation mask was 20.6s for the rare feature dataset, 21.3s for the rare time dataset, 19.8s for the mixture dataset, and 24.3s for the random dataset.

\subsubsection{Results} 
The results of the synthetic data experiments are presented in Fig.~\ref{fig1} to Fig.~\ref{fig4} at the end of this appendix. It is evident that CGS-Mask successfully identifies a greater number of truly salient inputs, resulting in higher AUR. Additionally, the figures demonstrate that CGS-Mask produces more intuitive and user-friendly masks compared to other saliency methods.

\subsection{More details on real-world data experiments}
\subsubsection{Data preprocessing}

In the MIMIC-III experiment, we followed the same processing steps as described in~\cite{tonekaboni2020went}, which includes data selection, preprocessing, and model training. The dataset used was the adult ICU admission data from the MIMIC-III dataset, containing de-identified EHRs for approximately 40,000 ICU patients at the Beth Israel Deaconess Medical Center~\cite{johnson2016mimic}. The input features for each patient consisted of three demographic factors (age, gender, ethnicity), eight vital measurements (HeartRate, SysBP, DiasBP, MeanBP, RespRate, SpO2, Glucose, Temp), and twenty lab measurements (ANION GAP, ALBUMIN, BICARBONATE, BILIRUBIN, CREATININE, CHLORIDE, GLUCOSE, HEMATOCRIT, HEMOGLOBIN, LACTATE, MAGNESIUM, PHOSPHATE, PLATELET, POTASSIUM, PTT, INR, PT, SODIUM, BUN, WBC) to predict mortality within the next 48 hours. Patients with missing data for all 48-hour blocks of specific features were excluded. Mean imputation was used for missing vital measurements, while forward imputation was applied for missing lab measurements. The resulting dataset was divided into a training set (65\%), a validation set (15\%), and a test set (20\%). We employed an RNN model with a single layer of 200 GRU cells for training.

The data for the remaining datasets was sourced from~\cite{tsai} and has undergone preprocessing. \textcolor{black}{The InceptionTime model~\cite{ismail2020inceptiontime}, based on a CNN architecture, was selected as the prediction model.} A summary of the four real-world datasets and the corresponding model $f$ for each upstream prediction task can be found in Table~\ref{summaryData}.

\begin{table}[t]
    \renewcommand{\arraystretch}{1.2}
    \setlength{\abovecaptionskip}{10pt}
    \centering
    \resizebox{0.95\columnwidth}{!}{
    \begin{tabular}{ccccc}
        \specialrule{1pt}{0pt}{0pt}
          &MIMIC-III & LSST  & NATOPS & AE\\
        \specialrule{1pt}{0pt}{0pt}
        Number of  Samples inTrain Size          & 18392  & 2459    & 180  & 95  \\
        Number of Samples in Test Size            & 4598   & 2466    & 180  & 42  \\
        Number of time steps        &48      & 36      & 51   & 144  \\
        Number of features    &31      & 6       & 24   & 24\\
        Architecture of $f$        &RNN     & CNN     & CNN  & CNN\\
        Evaluation of $f$   &0.79    & 0.71    & 0.95 & 2.49\\
         \specialrule{1pt}{0pt}{0pt}
    \end{tabular}}
    \caption[short title]{Summary for the real-world data sets and the results.}
    \label{summaryData}
\end{table}

\subsubsection{Metrics}
\paragraph{$\mathbf{\Delta_{p}^{M}}$}Consider a classifier $f$ that maps the input $\mathbf{X}$ to $f(\mathbf{X}) = \mathbf{P} = [P_{1},P_{2},...P_{C}]$, where $P_{i}$ is the probability of $\mathbf{X}$ being class $i$. We define the class function as follows:
\begin{equation}
Class(P_{i})=\left\{
\begin{array}{cl}
1  &  P_{i} = max(\mathbf{P}) \\
0  &  else \\
\end{array} \right.
\label{class}
\end{equation}
For a data set with $N$ samples, to measure the shift in the classifier’s prediction caused by the perturbation of the input, we use the binary cross-entropy:

\begin{equation}
    \begin{split}
        CE = -\frac{1}{N}\sum_{n=1}^{N}\sum_{i=1}^{C} Class [f(\mathbf{X}_{n})_{i}]\times \log f(\widehat{\mathbf{X}_{n}})_{i}
    \end{split}
\label{ce}
\end{equation}

Since \emph{MIMIC-III}, \emph{LSST}, and \emph{NATOPS} are all classification task data sets, we used Equation~\ref{ce} on them. For \emph{AE}, we used mean square error to measure the performance of masks.

For a fair comparison, when comparing $\Delta_{p}^{\mathbf{M}}$, each saliency method selected 10\% most important inputs as saliency regions, and the saliency values were all set to one. \textcolor{black}{For \emph{MIMIC-III} data set, we further investigated saliency regions by considering different percentages of the input, specifically 10\%, 20\%, 30\%, 40\%, 50\%, and 60\%. In each case, the saliency values for these regions were uniformly set to one.} 

\paragraph{$\mathcal{D}_{\mathbf{M}}$ and $\mathcal{E}_{\mathbf{M}}$} 

We preserved the original saliency value of the 10\% selected regions adopted by saliency methods and compared their $\mathcal{D}_{\mathbf{M}}$ and $\mathcal{E}_{\mathbf{M}}$. 
$\mathcal{D}_{\mathbf{M}}$ and $\mathcal{E}_{\mathbf{M}}$ calculation followed the same Equation (\ref{dm}) and Equation (\ref{em}).

\subsubsection{Mask generation and optimization}

For CGS-Mask, we need to set the length range of strips and the number of strips according to the input data size. Specifically, for \emph{MIMIC-III}, the strip length $l\in[2, 6]$, the number of strips was set to 35, and the number of points covered by the strips was 149. For \emph{LSST}, the strip length $l\in[2, 4]$, the number of strips was set to 7, and the number of points covered by the strips was 23. For \emph{NATOPS}, the strip length $l\in[2, 4]$, the number of strips was set to 40, and the number of points covered by the strips was 123. For \emph{AE}, the strip length $l\in[2, 4]$, the number of strips was set to 115, and the number of points covered by the strips was 346. We still need to set the parameters in Algorithm 1 described in the main body of our paper to optimize the strip masks. Specifically, for \emph{MIMIC-III}, we set \textcolor{black}{$N = 100$, $m=n=10$, $P_{c} = 0.7$, $P_{m}=0.1$, $P_{t} = 0.1$.} For \emph{LSST}, we set \textcolor{black}{$N = 100$, $m=n=10$, $P_{c} = 0.7$, $P_{m}=0.1$, $P_{t} = 0.1$.} For \emph{NATOPS}, we set \textcolor{black}{$N = 200$, $m=n=10$, $P_{c} = 0.7$, $P_{m}=0.1$, $P_{t} = 0.1$.} For \emph{AE}, we set \textcolor{black}{$N = 200$, $m=n=10$, $P_{c} = 0.7$, $P_{m}=0.1$, $P_{t} = 0.1$.}

\subsubsection{\textcolor{black}{ Runtime}}
We utilized the same running environment as previously described. \textcolor{black}{All reported data points are an average of 500 executions}. In the \emph{MIMIC-III} dataset, the average time to generate a mask for a given sample was 5.32 seconds. In the \emph{LSST} dataset, it took an average of 5.03 seconds to deliver a mask for a given sample. For the \emph{NATOPS} dataset, the average time to deliver a mask for a given sample was 7.22 seconds. Lastly, in the \emph{AE} dataset, it took an average of 10.03 seconds to generate a mask for a given sample.

\subsubsection{Results} 
The results of the real-world data experiments are depicted in Fig.~\ref{fig5} to Fig.~\ref{fig6} presented at the end of this appendix. The figures show that the top 10\% most important inputs as determined by all saliency methods. The findings demonstrate that CGS-Mask effectively identifies salient input regions and provides intuitive explanations to users. 

\textcolor{black}{Specifically, for \emph{MIMIC-III} experiment, the best fitness value progression from 0 to 100 generations (in intervals of 10) is as follows: $0.003\rightarrow 0.025\rightarrow 0.044\rightarrow 0.062\rightarrow 0.081\rightarrow 0.097\rightarrow  0.109\rightarrow 0.117\rightarrow 0.121\rightarrow  0.125\rightarrow 0.125$. This pattern clearly indicates convergence, with the value increasing and eventually stabilizing, unlike a random algorithm where the fitness value would be noisy. }

\subsubsection{Case study}
To gain a practical understanding of how the strip mask operates, we utilized the time series data of gesture changes in NATOPS as a demonstration of CGS-Mask. We selected specific data samples labeled as "I have command" based on the movement of the right hand and elbow. We applied CGS-Mask to explain four feature combinations: 1) "Hand tip right, X coordinate," 2) "Thumb right, X coordinate," 3) "Elbow right, Y coordinate," and 4) "Thumb right, Y coordinate." As depicted in Fig.~\ref{case1}, CGS-Mask effectively delineated the stages of the gesture's change. In this instance, the strip mask accurately captured the temporal continuity of the features, and its binary values proved suitable for practical utilization.

\section{Further details on pilot user study}
We conducted a survey research using an electronic questionnaire that comprised both qualitative and quantitative experiments. The questionnaire was distributed to a wide range of participants, including computer science students, primary school kids, retired individuals, and randomly selected users from the Internet. This diverse sample size ensured a comprehensive representation of different age groups and backgrounds. In total, we collected 254 valid questionnaires, further enhancing the reliability and credibility of our findings. \textcolor{black}{Additionally, we conducted a supplementary survey involving 54 participants with medical backgrounds, who are experts in their respective fields. This further validation aimed to ascertain the effectiveness of CGS-Mask in specific domains.}

\subsection{Qualitative Experiment}
In the qualitative experiment, participants were presented with a ranking question consisting of six saliency masks obtained from different methods. The masks were not labeled with specific methods to avoid bias. Participants were asked to rank the masks based on their ability to help quickly distinguish salient regions (green) from non-salient regions (red). The rankings were assigned based on participants' subjective evaluations, with higher rankings indicating masks that were more effective in fulfilling the task. Table~\ref{QualitativeResult} presents the distribution of participants' rankings for masks generated by different methods.

\begin{table}[t]
    \renewcommand{\arraystretch}{1.2}
    \setlength{\abovecaptionskip}{10pt}
    \centering
    \resizebox{0.95\columnwidth}{!}{
    \begin{tabular}{cccccccc}
        \specialrule{1pt}{0pt}{0pt}
          &Rank \#1 & Rank \#2  & Rank \#3 & Rank \#4 & Rank \#5 & Rank \#6 & Average Rank \\
        \specialrule{1pt}{0pt}{0pt}
        CGS-Mask & 168 & 37 & 16 & 5 & 5 & 23  & $\mathbf{1.86}$ \\
        Dynamask & 14 & 155 & 45 & 13 & 23 & 4  & 2.56 \\
        FO & 28 & 19 & 93 & 83 & 17 & 14  & 3.33 \\
        AFO& 8 & 14 & 12 & 45 & 99 & 76  & 4.74 \\
        FIT & 22 & 14 & 18 & 39 & 70 & 91  & 4.55 \\
        DeepLIFT & 14 & 15 & 70 & 69 & 40 & 46  & 3.96 \\
         \specialrule{1pt}{0pt}{0pt}
    \end{tabular}}
    \caption{The number of subjects selected for different rankings for masks delivered by different methods.}
    \label{QualitativeResult}
\end{table}

\subsection{Quantitative Experiment}
In the quantitative experiment, participants were presented with masks containing specific saliency scores, including binary masks (with values of 0 or 1) and numerical masks (with real values). The participants were instructed to identify the most important feature from each mask, where importance was defined as the sum of all the scores in the corresponding row. We recorded the time taken by each participant to complete each mask. A shorter time indicated that the mask facilitated faster identification of the most important features. The experiment included three questions (Q1, Q2, and Q3), with \textcolor{black}{Q2 featuring a binary mask and Q1 and Q3 featuring numerical masks.} To eliminate any potential bias caused by the order of questions, Q2 (with the binary mask) was presented as the second question. We calculated the average time taken and the average accuracy rate for answering Q1 and Q3. \textcolor{black}{The results have been presented in the main body of our paper.}

\subsection{Experiment on experts with domain knowledge}
To assess the interpretability of masks generated by various saliency methods among users with domain knowledge, we conducted a pilot user study involving 54 individuals with medical backgrounds. 
Participants were asked to evaluate the saliency masks obtained from six different methods based on their effectiveness in understanding salient features and temporal relevance. Our study revealed that 42.6\% of the users preferred CGS-Mask due to its intuitive interpretation of time series models, which was the highest preference among the six methods. The distribution of participants selecting different saliency methods is shown in Fig.~\ref{pie}.

\begin{figure*}[t]
\centering
\includegraphics[width=1.9\columnwidth]{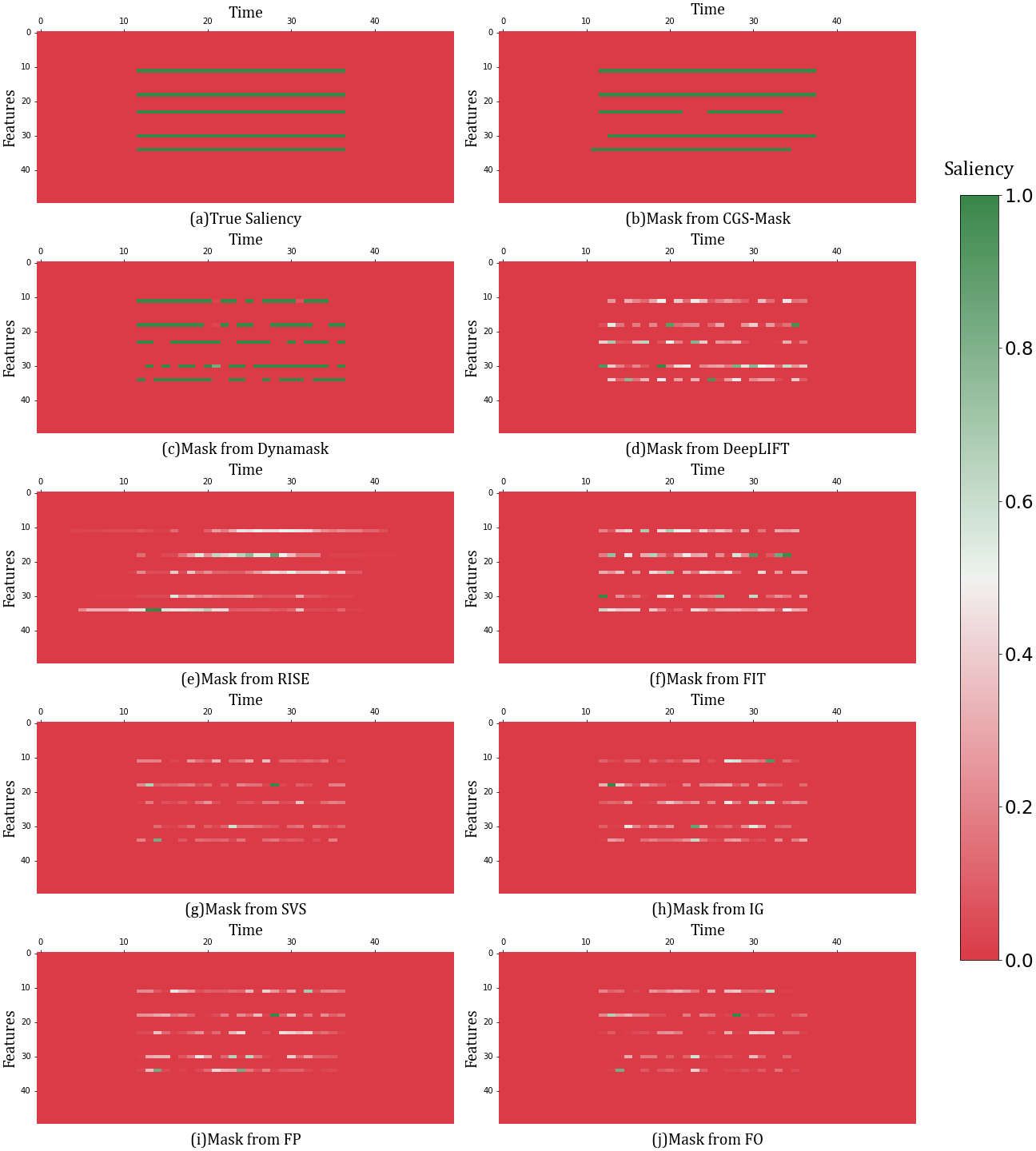} 
\caption{Nine explanation masks were used to analyze synthetic data from \emph{rare feature} experiments. The classification accuracy in these cases largely depends on the specific features of the data. For example, patients may exhibit abnormal blood pressure, respiration rate, heart rate, oxygen saturation levels, and altered mental status changes before experiencing shock. The ground truth results, as shown in Fig.~\ref{fig1}(a), serve as a reliable benchmark for evaluating the efficacy of the generated masks. Notably, among the nine explanation masks analyzed, the CGS-Mask generated an approximation that closely matched the ground truth, as shown in Figure~\ref{fig1}(b). Conversely, the masks shown in Fig.~\ref{fig1}(c)$-$(j) demonstrated either fragmented or vague results, obscuring meaningful indications and making them less intuitive for users.}
\label{fig1}
\end{figure*}

\begin{figure*}[t]
\centering
\includegraphics[width=1.9\columnwidth]{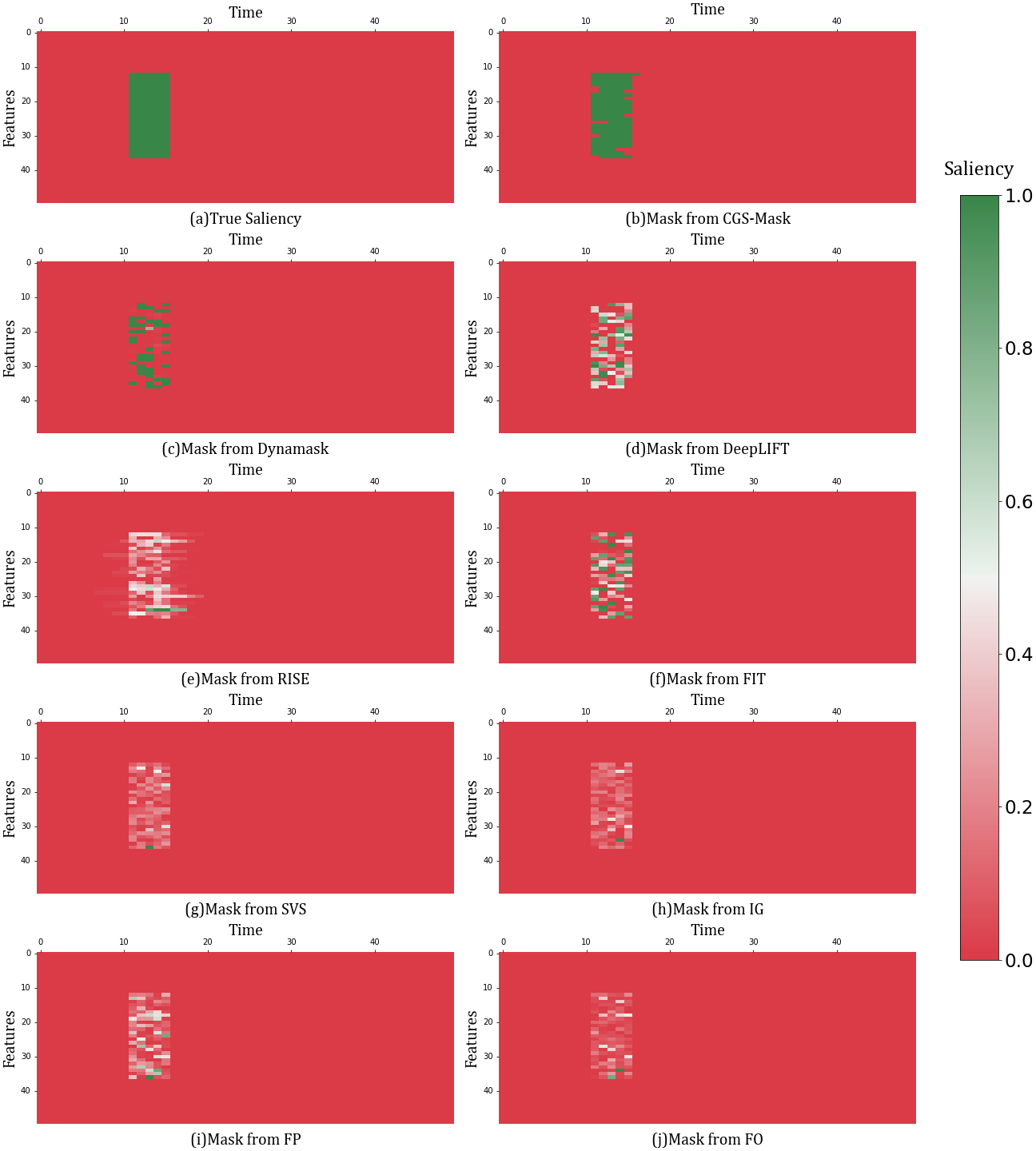} 
\caption{
Nine explanation masks were used to analyze synthetic data obtained from \emph{rare time} experiments. In this scenario, the classification of data is largely depends on the changes in each feature quantity over a specific short period. For instance, a rainstorm is typically accompanied by wind, lightning, and thunder. The ground truth outcomes, depicted in Figure~\ref{fig2}(a), serve as a reliable benchmark for evaluating the efficacy of the generated masks. Notably, among the nine explanation masks analyzed, the CGS-Mask generated an approximation that closely corresponded to the ground truth, as depicted in Figure~\ref{fig2}(b). By contrast, the masks shown in Fig.~\ref{fig2}(c)-(j) demonstrated either fragmented or vague results, obscuring meaningful indications and making them less intuitive for users.}
\label{fig2}
\end{figure*}

\begin{figure*}[t]
\centering
\includegraphics[width=1.9\columnwidth]{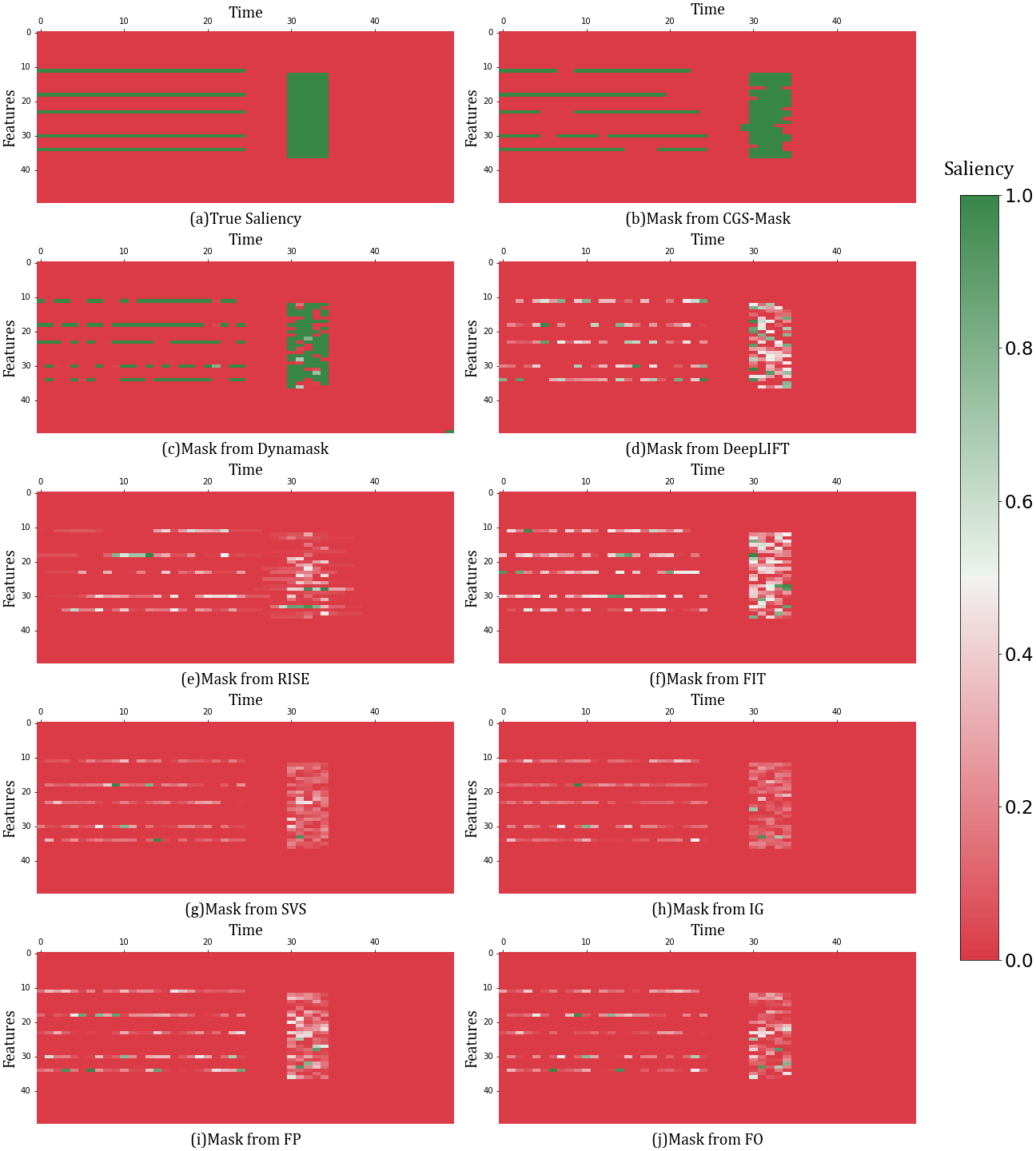} 
\caption{Nine explanation masks were used to analyze synthetic data obtained from mixture experiments. The experiments were designed to combine the cases depicted in Fig.~\ref{fig1} and Fig.~\ref{fig2}. The ground truth outcomes, as depicted in Fig.~\ref{fig3}(a), were utilized as a reliable benchmark to assess the effectiveness of the generated masks. The CGS-Mask produced an approximation that closely corresponded to the ground truth, as shown in Fig.~\ref{fig3}(b). By contrast, the masks shown in Fig.~\ref{fig3}(c)-(j) demonstrated either fragmented or vague results, obscuring meaningful indications and making them less intuitive for users.}
\label{fig3}
\end{figure*}

\begin{figure*}[t]
\centering
\includegraphics[width=1.9\columnwidth]{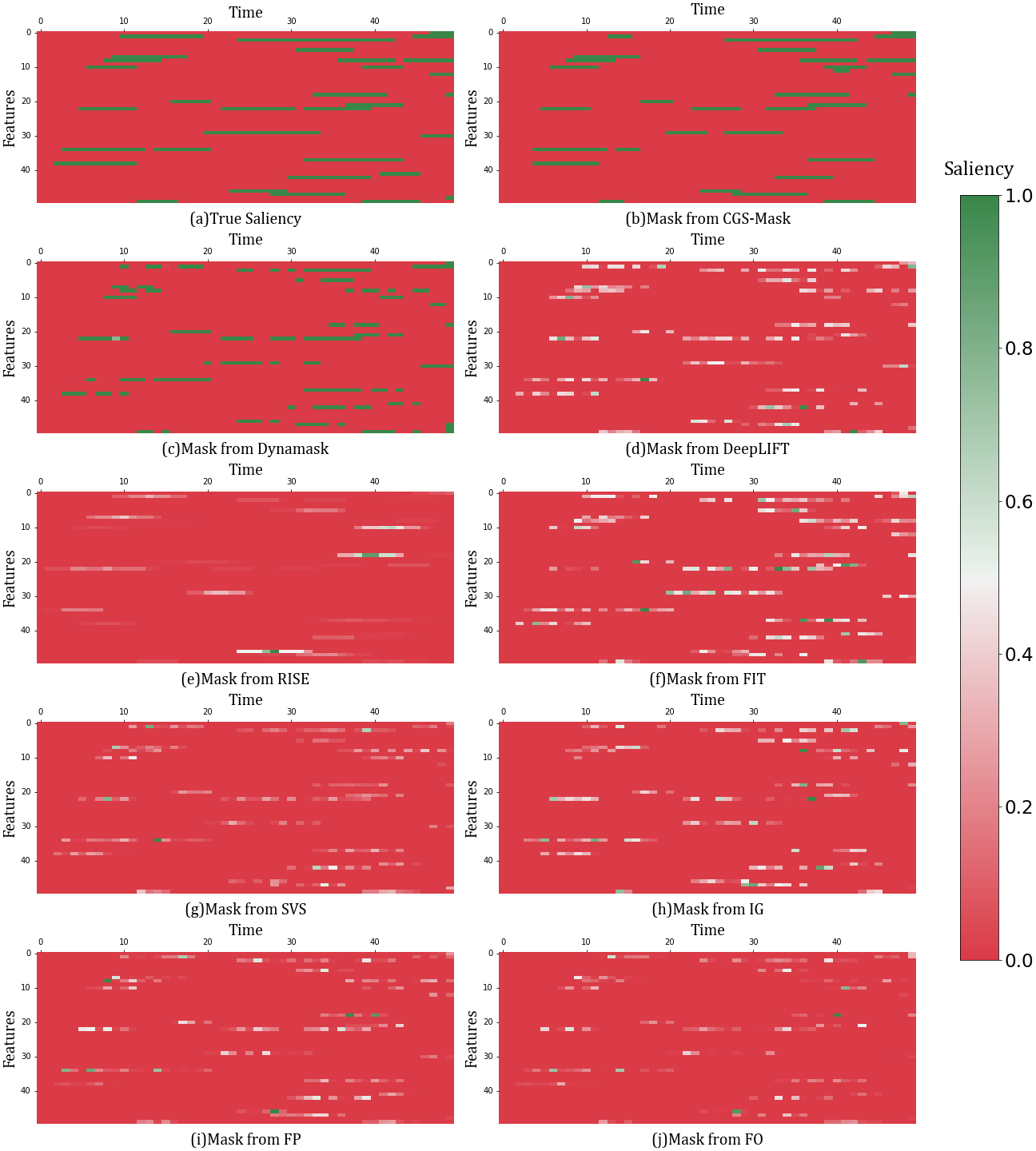} 
\caption{Nine explanation masks were employed to analyze synthetic data obtained from \emph{random} experiments. The present example pertains to a stochastic phenomenon that expounds information. It is closer to the actual application scenario, where the time series feature information, which has a crucial impact on the model, may be dispersed in different features and different periods. The ground truth outcomes, illustrated in Fig.~\ref{fig4}(a), serve as a reliable benchmark for evaluating the efficacy of the generated masks. It is noteworthy that among the nine explanation masks assessed, the CGS-Mask produced an approximation that closely corresponded to the ground truth, as depicted in Fig.~\ref{fig4}(b). In contrast, the masks shown in Fig.~\ref{fig4}(c)-(j) demonstrated either fragmented or vague results, obscuring meaningful indications and making them less intuitive for users.}
\label{fig4}
\end{figure*}

\begin{figure*}[t]
\centering
\includegraphics[width=1.9\columnwidth]{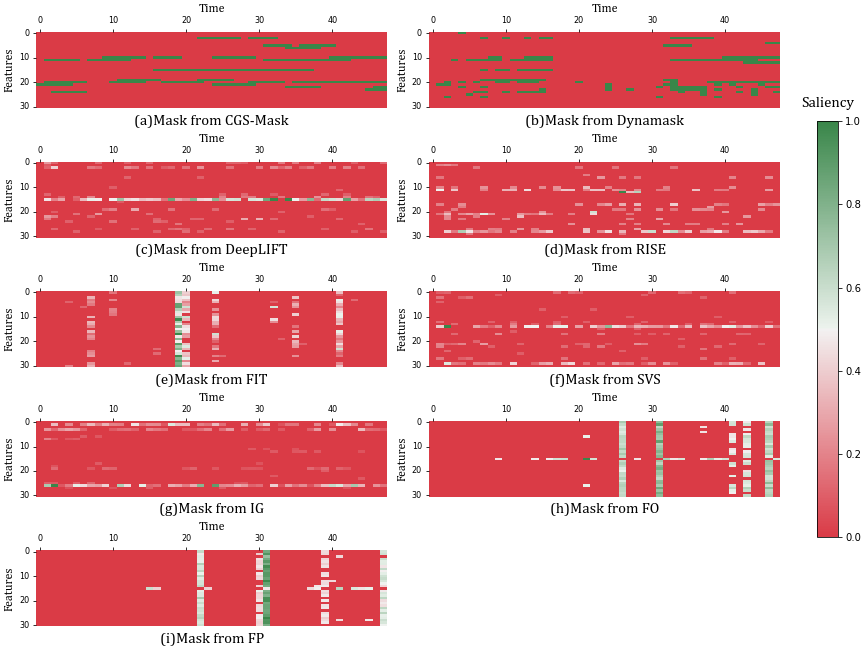} 
\caption{Nine explanation masks were employed to analyze data from \emph{MIMIC-III}, covering the 48 hours preceding a patient's death. The green strips in Fig.~\ref{fig5}(a) reveal significant features that indicate patient outcomes. Specifically, a decrease in blood pressure, tachycardia of heart rate, and rapid respiration are indicative of an impending risk of death, thereby enabling timely intervention by ICU (Intensive Care Unit) physicians. However, the remaining masks cannot distinguish the periods and features contributing to this outcome, as observed in Fig.~\ref{fig5}(c)$-$(i).}
\label{fig5}
\end{figure*}

\begin{figure*}[t]
\centering
\includegraphics[width=1.9\columnwidth]{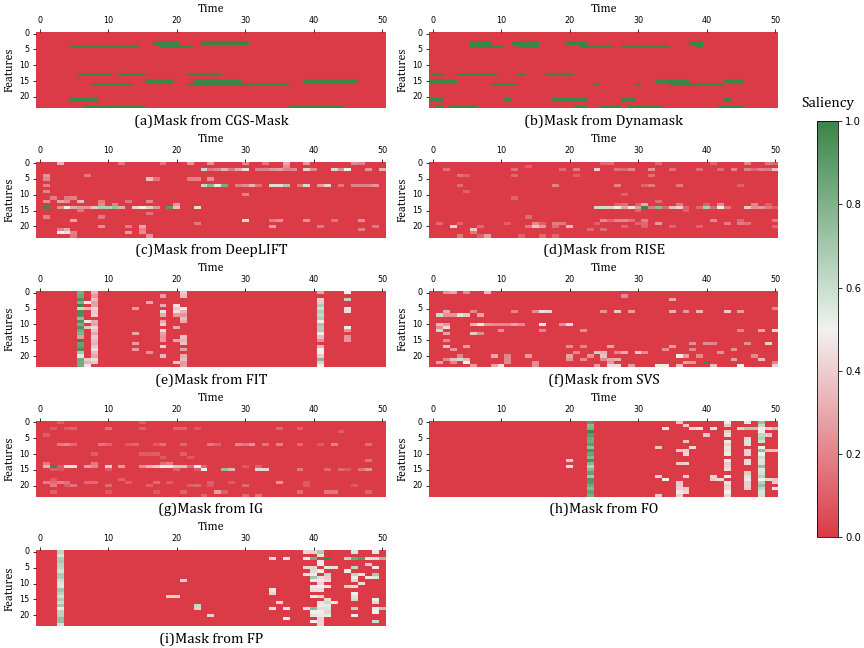} 
\caption{Nine explanation masks were employed to analyze data from the \emph{NATOPS}, explaining the classification of the movement "I have command." As is well known, a particular gesture can be accurately classified and inferred by the continuous changes of features such as "fingertip right, x-coordinate," "thumb right, x-coordinate," "elbow right, y-coordinate," and "thumb right, y-coordinate" over a certain period. CGS-Mask can clearly identify these consecutive changes, as illustrated in Fig.~\ref{fig7}(a). Compared to CGS-Mask, the other masks either fail to clearly define the essential features at which moments or fail to produce chronologically continuous results. A more concrete example is shown in Fig.~\ref{case1}.
}
\label{fig7}
\end{figure*}

\begin{figure*}[t]
\centering
\includegraphics[width=1.9\columnwidth]{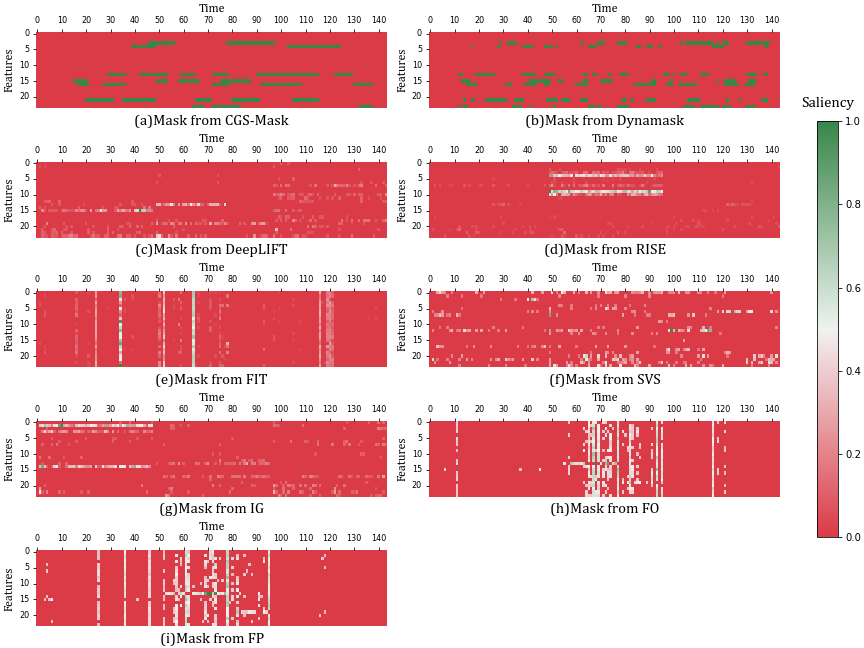} 
\caption{Nine explanation masks were used to analyze the \emph{AE (Appliances Energy)} data to understand the energy consumption of the appliances within a low-energy building. The interpretation of the data reveals that T1 (Temperature in the kitchen area), T2 (Temperature in the living room area), T6 (Temperature outside the building), and RH\_5 (Humidity in the bathroom) are the key factors that play an essential role in measuring the use of Appliances Energy. In contrast to CGS-Mask, the other masks' interpretation results fail to reveal the necessary time periods of these critical regions.}
\label{fig8}
\end{figure*}

\begin{figure*}[t]
\centering
\includegraphics[width=1.9\columnwidth]{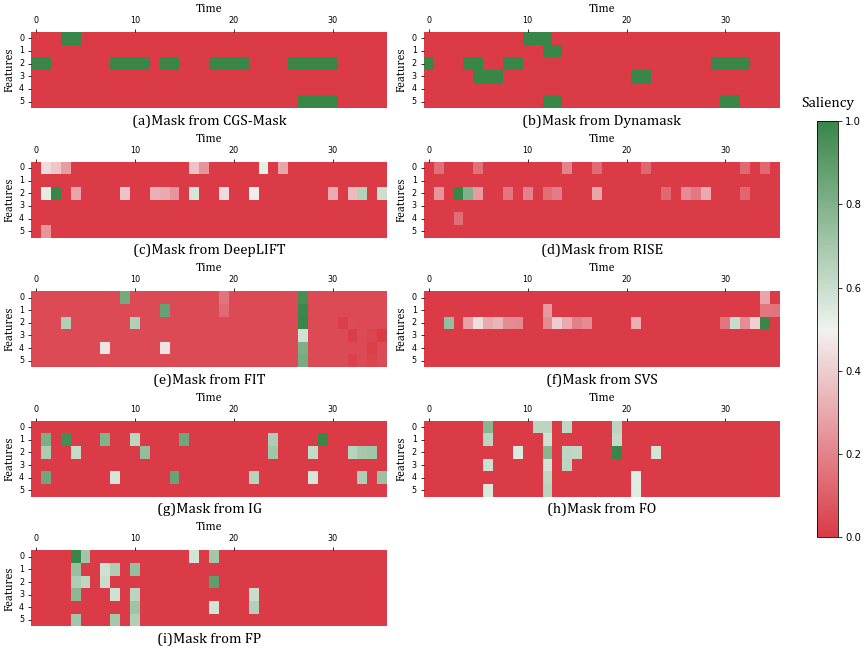} 
\caption{Nine explanation masks were employed to analyze data from the \emph{LSST} data set to give explanations for classifying astronomical sources that vary over time into different classes. Upon examining the CGS-Mask sample, it is evident that the flux feature, which measures the brightness in the band of observation, plays a crucial role in determining the class of the sample.}
\label{fig6}
\end{figure*}

\newpage

\begin{figure*}[t]
\centering
\includegraphics[width=1.9\columnwidth]{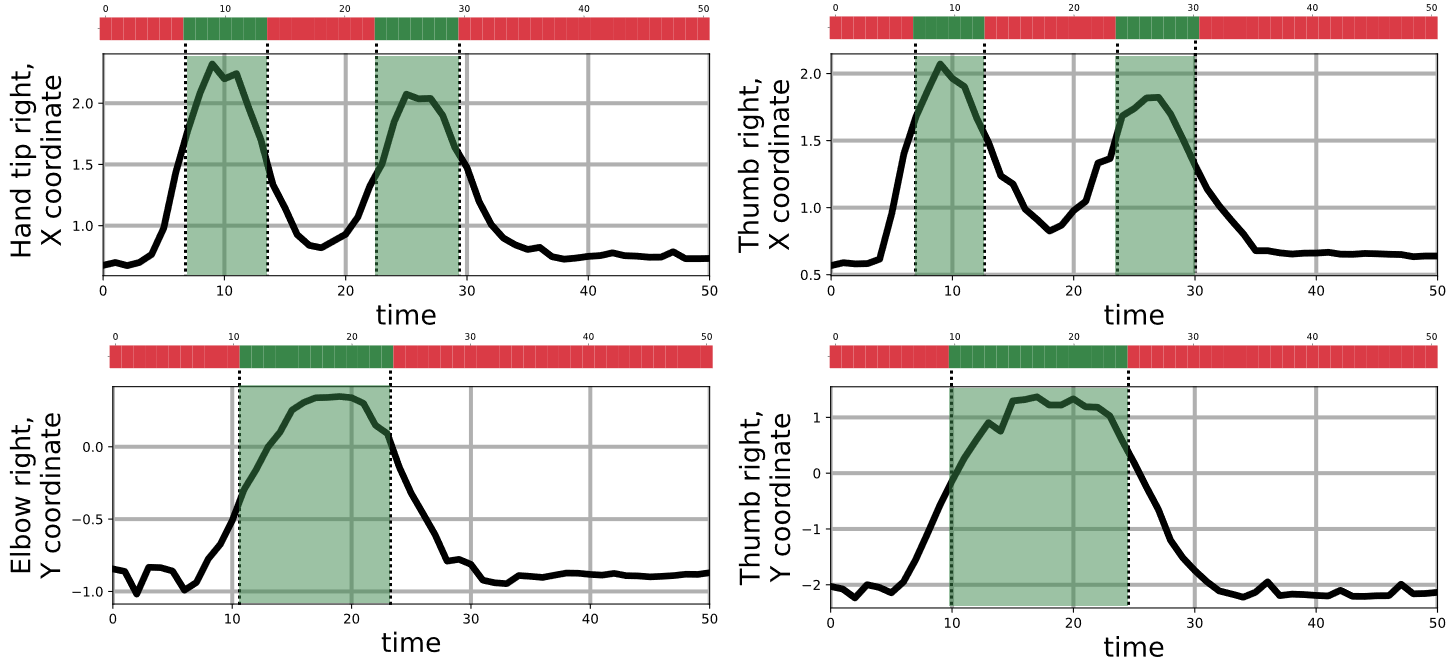} 
\caption{Case study on \emph{NATOPS} data set with four selected features. The green means the data is salient while the red means the opposite.}
\label{case1}
\end{figure*}

\begin{figure*}[t]
\centering
\includegraphics[width=1\columnwidth]{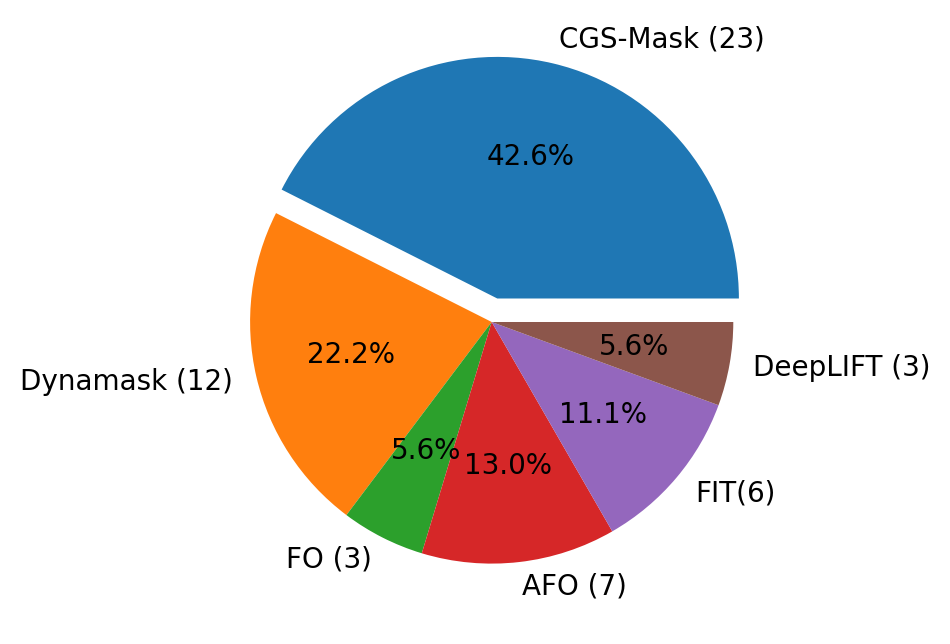} 
\caption{\textcolor{black}{Results of the pilot user study conducted with participants from a medical domain background.}}
\label{pie}
\end{figure*}

\clearpage

\end{document}